\documentclass[accepted]{uai2022} 

\usepackage[american]{babel}

\usepackage{natbib} 
\usepackage{mathtools} 
\usepackage{booktabs} 
\usepackage{tikz} 

\usepackage[utf8]{inputenc} 
\usepackage[T1]{fontenc}    
\usepackage{xr-hyper}
\usepackage{hyperref}       
\usepackage{url}            
\usepackage{amsfonts}       
\usepackage{nicefrac}       
\usepackage{microtype}      
\hypersetup{
    colorlinks=true,
    allcolors=brown
}

\usepackage{algorithm}
\usepackage{algorithmic}
\usepackage{amsmath}
\usepackage{amsthm}
\usepackage{amssymb}
\usepackage{physics} 
\usepackage{numprint}
\usepackage{bbm}
\usepackage{dsfont}
\usepackage{float}
\usepackage{graphicx} 
\usepackage{subcaption}
\usepackage[font={small}]{caption}

\makeatletter
\newcommand*{\addFileDependency}[1]{
  \typeout{(#1)}
  \@addtofilelist{#1}
  \IfFileExists{#1}{}{\typeout{No file #1.}}
}
\makeatother

\newcommand*{\myexternaldocument}[1]{%
    \externaldocument{#1}%
    \addFileDependency{#1.tex}%
    \addFileDependency{#1.aux}%
}

\newcommand{\added}[1]{{\color{black}#1}}

\usepackage{amsmath,amsfonts,bm}

\def\eqref#1{equation~\ref{#1}}

\def\1{\bm{1}}

\DeclareMathAlphabet{\mathsfit}{\encodingdefault}{\sfdefault}{m}{sl}
\SetMathAlphabet{\mathsfit}{bold}{\encodingdefault}{\sfdefault}{bx}{n}

\newcommand{\E}{\mathbb{E}}

\newcommand{\R}{\mathbb{R}}

\newcommand{\ECE}{{\mathrm{ECE}}}
\newcommand{\LCE}{\mathrm{LCE}}
\newcommand{\MLCE}{\mathrm{MLCE}}
\newcommand{\method}{LoRe}

\newcommand{\hp}{{\hat{p}}}

\newtheorem{theorem}{Theorem}
\newtheorem{lemma}{Lemma}

\newcommand{\Xc}{{\mathcal{X}}}
\newcommand{\Yc}{{\mathcal{Y}}}

\newcommand{\Eb}{{\mathbb{E}}}

\myexternaldocument{luo_333-supp}
\title{Local Calibration: Metrics and Recalibration}
\author[1]{\href{mailto:<rsluo@stanford.edu>?Subject=Your UAI 2022 paper}{Rachel Luo}{}$^*$}
\author[2]{\href{mailto:<abhatnagar@salesforce.com>?Subject=Your UAI 2022 paper}{Aadyot Bhatnagar}{}$^*$}
\author[2]{Yu Bai}
\author[1]{Shengjia Zhao}
\author[2]{Huan Wang}
\author[2]{Caiming Xiong}
\author[1,2]{Silvio Savarese}
\author[1,2]{Stefano Ermon}
\author[1]{Edward Schmerling}
\author[1]{Marco Pavone}
\affil[1]{%
    Stanford University\\
    Stanford, California, USA
}
\affil[2]{%
    Salesforce AI Research\\
    Palo Alto, California, USA
}

\begin{document}
\maketitle

\begin{abstract}
Probabilistic classifiers output confidence scores along with their predictions, and these confidence scores should be calibrated, i.e., they should reflect the reliability of the prediction. Confidence scores that minimize standard metrics such as the expected calibration error (ECE) accurately measure the reliability \textit{on average} across the entire population. However, it is in general impossible to measure the reliability of an \textit{individual} prediction. In this work, we propose the local calibration error (LCE) to span the gap between average and individual reliability. For each individual prediction, the LCE measures the average reliability of a set of similar predictions, where similarity is quantified by a kernel function on a pretrained feature space and by a binning scheme over predicted model confidences. We show theoretically that the LCE can be estimated sample-efficiently from data, and empirically find that it reveals miscalibration modes that are more fine-grained than the ECE can detect. Our key result is a novel {\textbf{lo}cal \textbf{re}calibration} method \method{}, to improve confidence scores for individual predictions and decrease the LCE. Experimentally, we show that our recalibration method produces more accurate confidence scores, which improves downstream fairness and decision making on classification tasks with both image and tabular data. 
\end{abstract}

\def\thefootnote{*}\footnotetext{Equal contribution. Order decided by coin flip.}

\section{Introduction}
\label{sec:intro}

Uncertainty estimation is extremely important in high stakes decision-making tasks. For example, a patient wants to know the probability that a medical diagnosis is correct; an autonomous driving system wants to know the probability that a pedestrian is correctly identified. Uncertainty estimates are usually achieved by predicting a probability along with each classification. Ideally, we want to achieve individual calibration, i.e., we want to predict the probability that each sample is misclassified. 

However, each sample is observed only once for most datasets (e.g., image classification datasets do not contain identical images), making it impossible to estimate, or even define, the probability of incorrect classification for individual samples. 
Because of this, commonly used metrics such as the expected calibration error (ECE) measure the gap between a classifier's confidence and accuracy \textit{averaged} across the entire dataset. Consequently, ECE can be accurately estimated but does not measure the reliability of individual predictions. 

In this work, we propose the local calibration error (LCE), a calibration metric that spans the gap between fully global (e.g., ECE) and fully individual calibration. Motivated by the success of kernel-based locality in other fields such as fairness (where similar individuals should be treated similarly)~\citep{dwork2012fairness, pleiss17-fairness} and causal inference (where matching techniques are used to find similar neighboring samples)~\citep{stuart2010matching}, we approximate the probability of misclassification for an individual sample by computing the average classification error over similar samples, where similarity is measured by a kernel function in a pre-trained feature space and a binning scheme over predicted confidences. Intuitively, two samples are similar if they are close in a pretrained feature space and have similar predicted confidence scores.   
By choosing the bandwidth of the kernel function, we can trade off estimation accuracy and individuality: when the bandwidth is very large, 
we recover existing global calibration metrics; when the bandwidth is small, we approximate individual calibration. We choose an intermediate bandwidth, so our metric can be accurately estimated, and provides some measurement on the reliability of individual predictions. 

Theoretically, we show that the LCE can be estimated with polynomially many samples if the kernel function is bounded. Empirically, we also show that for intermediate values of the bandwidth, the LCE can be accurately estimated and reveals modes of miscalibration that global metrics (such as ECE) fail to uncover. 

In addition, we introduce a non-parametric, post-hoc \textbf{lo}calized \textbf{re}calibration method (\method), for lowering the LCE. 
Empirically, \method{} improves fairness by achieving low calibration error on all potentially sensitive subsets of the data, such as racial groups. Notably, it can do so without any prior knowledge of those groups, and is more effective than global methods at this task. 
In addition, our recalibration method improves decision making when there is a ``safe'' action that is selected whenever the predicted confidence is low. For example, an automated system which classifies tissue samples as cancerous should request a human expert opinion whenever it is unsure about a classification. In a simulation on an image classification dataset, we show that recalibrated prediction models more accurately choose whether to use the ``safe'' action, which improves the overall utility. 

In summary, the contributions of our paper are as follows. (1) We introduce a local calibration metric, the LCE, that is both easy to compute and can estimate the reliability of individual predictions. (2) We introduce a post-hoc localized recalibration method \method{}, that transforms a model's confidence predictions to improve the local calibration. (3) We empirically evaluate \method{} on several downstream tasks and observe that \method{} improves fairness and decision-making more than existing baselines.

\section{Background and Related Work}
\label{sec:background}

\subsection{Global Calibration Metrics}

Consider a classification task that maps from some input domain (e.g., images) $\Xc$ to a finite set of labels $\Yc = \lbrace 1, \cdots, m \rbrace$.
A classifier is a pair $(f, \hp)$ where $f: \Xc \to \Yc$ maps each input $x \in \Xc$ to a label $y \in \Yc$ and $\hp: \Xc \to [0, 1]$ maps each input $x$ to a confidence value $c$. Let $\Pr$ be a joint distribution on $\Xc \times \Yc$ (e.g., from which training or test data pairs $(x, y)$ are drawn). The classifier $(f, \hp)$ is perfectly calibrated~\citep{guo17-ts} with respect to $\Pr$ if for all $c \in [0, 1]$
\begin{equation}
    \Pr[f(X) = Y \mid \hp(X) = c] = c.
    \label{eq:perfect_calibration} 
\end{equation}

To numerically measure how well a classifier is calibrated, the most commonly used metric is the expected calibration error (ECE)~\citep{naeini2015obtaining, guo17-ts}, which measures the average absolute deviation from Eq.~\ref{eq:perfect_calibration} over the domain. 
In practice, given a finite dataset, the ECE is approximated by binning. The predicted confidences $\hp$ are partitioned into bins $B_1, \ldots, B_k$, and then a weighted average is taken of the absolute difference between the average confidence $\mathrm{conf}(B_i)$ and average accuracy $\mathrm{acc}(B_i)$ for each bin $B_i$: 
\begin{equation}
    \ECE(f, \hp) := \sum_{i = 1}^{k} \frac{\abs{B_i}}{N} \abs{\mathrm{conf}(B_i) - \mathrm{acc}(B_i)}.
    \label{eq:ece} 
\end{equation}

Similarly, the maximum calibration error (MCE)~\citep{naeini2015obtaining, guo17-ts} measures the average deviation from Eq.~\ref{eq:perfect_calibration} in the bin with the highest calibration error, and is defined as
\begin{equation}
    \mathrm{MCE}(f, \hp) := \max_{i} \abs{\mathrm{conf}(B_i) - \mathrm{acc}(B_i)}.
    \label{eq:mce} 
\end{equation}

\subsection{Existing Global Recalibration Methods}

Many existing methods apply a post-hoc adjustment that changes a model's confidence predictions to improve global calibration, including Platt scaling~\citep{platt99}, temperature scaling~\citep{guo17-ts}, isotonic regression~\citep{zadrozny02-isotonic}, and histogram binning~\citep{zadrozny01-hb}. These methods all learn a simple transformation from the original confidence predictions to new confidence predictions, and aim to decrease the expected calibration error (ECE). Platt scaling fits a logistic regression model; temperature scaling learns a single temperature parameter to rescale confidence scores for all samples simultaneously; isotonic regression learns a piece-wise constant monotonic function; histogram binning partitions confidence scores into bins $\lbrace [0, \epsilon), [\epsilon, 2\epsilon), \cdots, [1-\epsilon, 1] \rbrace$ and sorts each validation sample into a bin based on its confidence $\hp(x)$; it then resets the confidence level of all samples in the bin to match the classification accuracy of that bin. 

\subsection{Local Calibration}

Two notions of calibration that address some of the deficits of global calibration are class-wise calibration and group-wise calibration. Class-wise calibration groups samples by their true class label~\citep{kull19-dirichlet, Nixon2019MeasuringCI} and measures the average class ECE, while group-wise calibration uses pre-specified groupings (e.g., race or gender)~\citep{kleinberg16-groupcal, pleiss17-fairness} and measures the average group-wise ECE or maximum group-wise MCE. 

A few recalibration methods have been proposed for these notions of calibration as well. Dirichlet calibration~\citep{kull19-dirichlet} achieves calibration for groups defined by class labels, but does not generalize well to settings with many classes~\citep{zhao2021calibrating}. Multicalibration~\citep{hebert-johnson17-groupcal} achieves calibration for any group that can be represented by a polynomial sized circuit, but lacks a tractable algorithm. If the groups are known a priori, one can also apply global calibration methods within each group; however, this is impractical in many situations where the groups are not known for new examples at inference time. 

At an even more local level,~\cite{Zhao2020IndividualCW} looks at individual calibration in the regression setting and concludes that individual calibration is impossible to verify with a deterministic forecaster, and thus there is no general method to achieve individual calibration. 

\subsection{Kernel-Based Calibration Metrics}

\cite{kumar18a-mmce} introduces the maximum mean calibration error (MMCE), a kernel-based quantity that replaces the hard binning of the standard ECE estimator with a kernel similarity $k(\hp(x), \hp(x'))$ between the confidence of two examples. They further propose to optimize the MMCE directly in order to achieve better model calibration globally. ~\cite{widmann19-kce} extends their work and proposes the more general kernel calibration error. \added{\cite{zhang2020mixnmatch} and \cite{gupta2020calibration} also consider kernel-based calibration.} However, these methods only consider the similarity between model confidences $\hp(x), \hp(x')$, rather than the inputs $x, x'$ themselves.

\section{The Local Calibration Error}

Recall that commonly used metrics for calibration, such as the ECE or the MCE, are global in nature and thus only measure an \textit{aggregate} reliability over the entire dataset, making them insufficient for many applications. An ideal calibration metric would instead measure calibration at an individual level; however, doing so is impossible without making assumptions about the ground truth distribution~\citep{Zhao2020IndividualCW}. A localized calibration metric represents an adjustable balance between these two extremes. Ideally, such a metric should measure calibration at a local level (where the extent of the local neighborhood can be chosen by the user) and group similar data points together. 

In this section, we introduce the local calibration error (LCE), a kernel-based metric that allows us to measure the calibration locally around a prediction. Our metric leverages learned features to automatically group similar samples into a soft neighborhood, and allows the neighborhood size to be set with a hyperparameter $\gamma$.
We also consider only points with a similar model confidence as the prediction, so that similarity is defined in terms of distance both in the feature space and in model confidence.
Thus, the LCE effectively creates soft groupings that depend on the feature space; with a semantically meaningful feature space, these groupings correspond to useful subsets of the data. 
We then mention a few design choices and visualize LCE maps over a 2D feature space to show that we can use our metric to diagnose regions of local miscalibration. 

\subsection{Local Calibration Error Metric}

We propose a metric to measure calibration locally around a given prediction. The calibration of similar samples should be similar, so we use a kernel similarity function $k_{\gamma}: \Xc \times \Xc \to \mathbb{R}_+$, which provides similarity scores, to define soft local neighborhoods. $k_{\gamma}(x, x')$ has bandwidth $\gamma > 0$, which determines the extent of the local neighborhood --- as $\gamma$ increases, the neighborhood grows. Less similar (i.e., further away) samples $x'$ have less influence on the local calibration metric at $x$.
Also, as with the ECE and MCE (Eqs.~\ref{eq:ece} and \ref{eq:mce}), we use binning and consider only the points in the same confidence bin as $x$. Thus, the samples that influence the local calibration metric at $x$ are similar to $x$ in both features and model confidences. 

More formally, let $\phi: \Xc \to \mathbb{R}^d$ be a feature map that transforms an input to a feature vector, and let $k_{\gamma}$ be parameterized as  $k_{\gamma}(x, x') = g((\phi(x) - \phi(x')) / \gamma)$ for some Lipschitz function $g: \mathbb{R}^d \to \mathbb{R}_+$. Then given a data point $x \in \Xc$ and a classifier $(f, \hp)$, the {\em local calibration error} (LCE) of the model at $x$ is the expected difference between the model's confidence and accuracy on a randomly sampled data point $x' \sim \Pr$, weighted by the kernel similarity $k_{\gamma}(x, x')$. 

We say a probabilistic classifier $(\hat{p}, f)$ is \emph{perfectly locally calibrated} with respect to $k_\gamma$ if
\begin{align*}
    &\sup_{x\in {\rm supp}(\Pr)} \underbrace{\left(\E_{(x', y')\sim \Pr}\Big[ (\hat{p}(x') - \mathds{1}[f(x') = y']) \atop \cdot\ k_\gamma(x, x')~\Big\vert~\hat{p}(x') = \hat{p}(x)\Big] \right)}_{:= \mathrm{LCE}^\star_{\gamma}(x; f, \hat{p})} = 0.
\end{align*}
Similar to perfect calibration, perfect local calibration is achieved by the Bayes-optimal classifier. In general, perfect local calibration is a much stricter notion than perfect calibration due to localizing to each indiviudal data point $x$, and reduces to perfect calibration if $k_\gamma(x, x')\equiv 1$ is a trivial kernel. 

To define LCE on a finite dataset, we perform an additional binning on the confidence to deal with the conditioning.

Let $\mathcal{D} = ((x_1, y_1), \ldots, (x_N, y_N))$ be a dataset, and let $\beta(x) = \{i : \hp(x_i) \in B(\hp(x)) \}$ be the set of indices of the points in $\mathcal{D}$ occupying the same confidence bin as $x$. Then we can compute the LCE by
\begin{multline}
    \mathrm{LCE}_{\gamma}(x; f, \hp) = \\
    \abs{\frac{\sum_{i \in \beta(x)} (\hp(x_i) - \mathds{1}[f(x_i) = y_i]) k_\gamma(x, x_i) }{\sum_{i \in \beta(x)} k_\gamma(x, x_i)}}.
    \label{eq:lce}
\end{multline}
Note that the quantity $(\hp(x_i) - \mathds{1}[f(x_i) = y_i])$ is simply the difference between the confidence and the accuracy for sample $x_i$, and the denominator is a normalization term.

We then define the maximum local calibration error (MLCE) as 
\begin{equation}
    \MLCE_{\gamma}(f, \hp) := \max_{x} \mathrm{LCE}_{\gamma}(x; f, \hp).
    \label{eq:mlce}
\end{equation}

Intuitively, the LCE considers a neighborhood about a sample $x$ (as defined by the kernel $k_\gamma$ and the confidence bin $B$), and computes the kernel-weighted average of the difference between the confidence and accuracy for each sample in that neighborhood. 
Note that by changing the bandwidth $\gamma$, we can interpolate the LCE between an individualized calibration metric (as $\gamma \to 0$) and a global one (as $\gamma \to \infty$). Lemma~\ref{lemma:global} makes this more concrete under the assumption that $\lim_{\gamma \to \infty} k_{\gamma}(x, x') = 1$ (proof in Appendix \ref{appendix:lemma}). For example, the Laplacian and Gaussian kernels satisfy this condition.

\begin{lemma} \label{lemma:global}
As $\gamma \to \infty$, the MLCE converges to the MCE.
\end{lemma}

Theorem~\ref{theorem:slce_informal} shows that under certain regularity conditions, the finite-sample estimator $\mathrm{LCE}(x)$ converges uniformly and sample-efficiently to its true expected value $\mathrm{LCE}^{*}(x)$:

\begin{theorem}
\label{theorem:slce_informal}
(Informal) Let $\alpha \le \inf_{x\in\Xc} \Eb\left[k_\gamma(X, x) \mathds{1}[\hp(X) \in B(\hp(x))] \right]$ be a lower bound on the expectation of the kernel, and $d$ be the dimension of the kernel's feature space. If the sample size is at least $\widetilde{O}(d / \alpha^4\epsilon^2)$ where $\epsilon>0$ is a target accuracy level, then with probability at least $1-\delta$ we have
\begin{align*}
    \sup_{x \in \mathcal{X}} \abs{\mathrm{LCE}_\gamma(x; f, \hp) - \mathrm{LCE}^{*}_\gamma(x; f, \hp) } \le \epsilon.
\end{align*}
Here, $\widetilde{O}$ hides log factors of the form $\log(1/\alpha\gamma\delta\epsilon)$. In practice, $\alpha$ depends inversely on $\gamma$.
\end{theorem}

To summarize, the MLCE measures a worst-case individual calibration error as $\gamma \to 0$ (i.e., the effective neighborhood is very small) and converges to the global MCE metric as $\gamma \to \infty$ (i.e., the effective neighborhood is very large). In practice, one must pick intermediate values of $\gamma$ to balance a more local notion of calibration error with the sample efficiency of its estimation. 
A more formal statement and full proof of Theorem \ref{theorem:slce_informal} can be found in Appendix~\ref{appendix:lce}.

\subsection{Choice of Kernel and Feature Map} 
\label{section:choice-kernel}

In this work, we compute the LCE using 15 equal-width bins and use the Laplacian kernel
\begin{align*}
    k_{\gamma}(x, x') = \exp(-\frac{\norm{\phi(x) - \phi(x')}_1}{d \gamma}).
\end{align*}

Because distances in a high-dimensional input space (e.g., image data) may not be meaningful on their own, we evaluate the kernel on a feature representation of $x$ rather than on $x$ itself. 
Features learned from neural networks have proven useful for a wide range of tasks, and they have been shown to capture useful semantic features of their inputs \citep{huh16makes, chen20-simple, li20-sentence}. The kernel similarity term $k_{\gamma}(x, x')$ in the LCE thus leverages learned features to {\em automatically} capture rich subgroups of the data. For image data, we chose an Inception-v3 model as our feature map, since Inception features are widely accepted as useful and representative in many areas 
(e.g., for generative models~\citep{Salimans2016ImprovedTF}), \added{though other neural features can also be used (Appendix \ref{appendix:more_results})}. For tabular data, we used the final hidden layer of the neural network trained for classification.

\begin{figure} 
    \centering
    \includegraphics[width=0.8\linewidth,trim={0cm 0cm 0cm 0cm}]{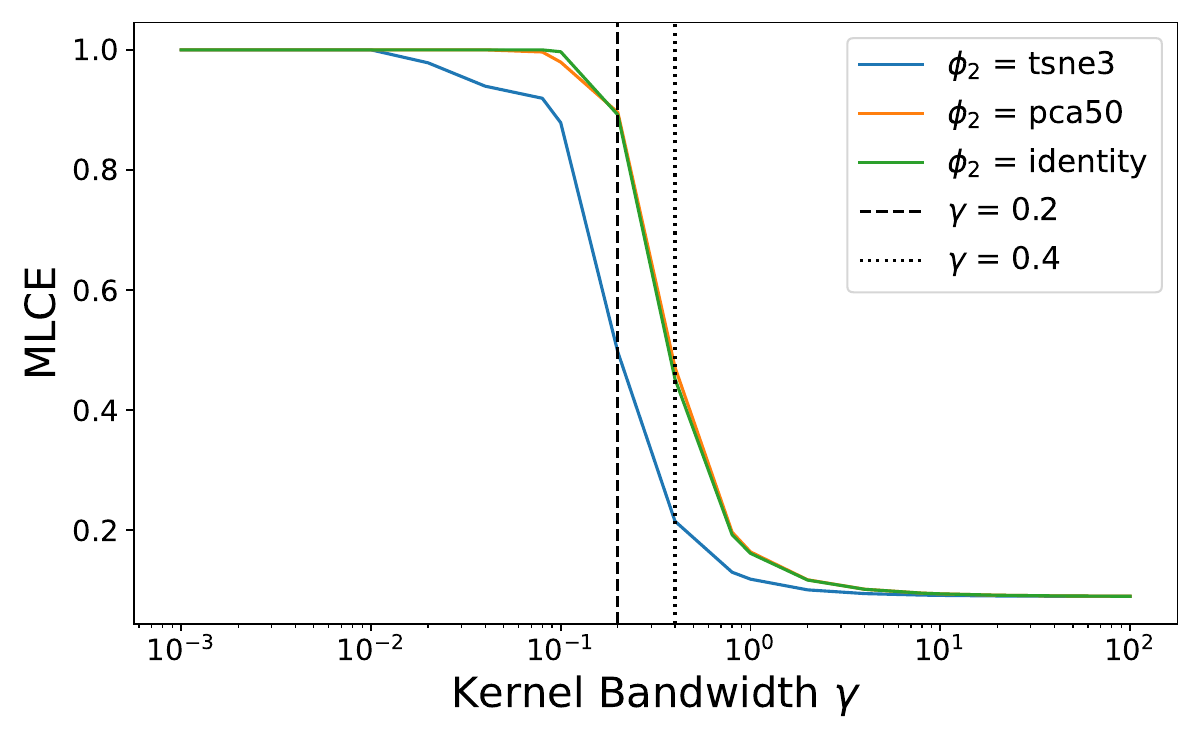}
    \caption{MLCE of a Resnet-50 classifier on the ImageNet test split, as a function of the kernel bandwidth $\gamma$. We use a Laplacian kernel with feature map $\phi_2 \circ \phi_1$, where $\phi_1: \mathcal{X} \to \mathbb{R}^{2048}$ is the Inception-v3 model's hidden layer. Blue: $\phi_2: \mathbb{R}^{2048} \to \mathbb{R}^3$ is t-SNE; orange: $\phi_2: \mathbb{R}^{2048} \to \mathbb{R}^{50}$ is PCA; green: $\phi_2(z) = z$ is the identity.}
    \label{fig:mlce_limits}
\end{figure}

In general, we also use t-SNE or PCA to reduce the dimension of the feature space. For example, the 2048-D Inception-v3 embedding is still very high-dimensional. \added{We report results using t-SNE to reduce the dimension to 2 or 3, as well as PCA to reduce the dimension to reduce the dimension to 50 for image data and 20 for tabular data.}
Thus the overall representation function is $\phi(x) = \phi_2(\phi_1(x))$, where $\phi_1$ maps from the inputs to the neural features, and $\phi_2$ reduces the feature space dimension. 

Figure~\ref{fig:mlce_limits} plots the MLCE as a function of the kernel bandwidth for an ImageNet classification task. Note that when $\gamma$ is small, the MLCE is 1 (a worst-case individual calibration error), and when $\gamma$ is large, the MLCE approaches the global MCE. 
To obtain a single summary statistic describing the local calibration error, we can view this plot and pick a value of $\gamma$ between the limiting behaviors. We find that $\gamma = 0.2$ \added{and $\gamma = 0.4$} are good intermediate points for the 3-D t-SNE \added{and 50-D PCA} features, respectively (Figure~\ref{fig:mlce_limits}).

\subsection{Local Calibration Error Visualizations}
\begin{figure*}[h]
    \centering
    \includegraphics[width=0.87\textwidth, trim={0cm 0cm 0cm 0cm}]{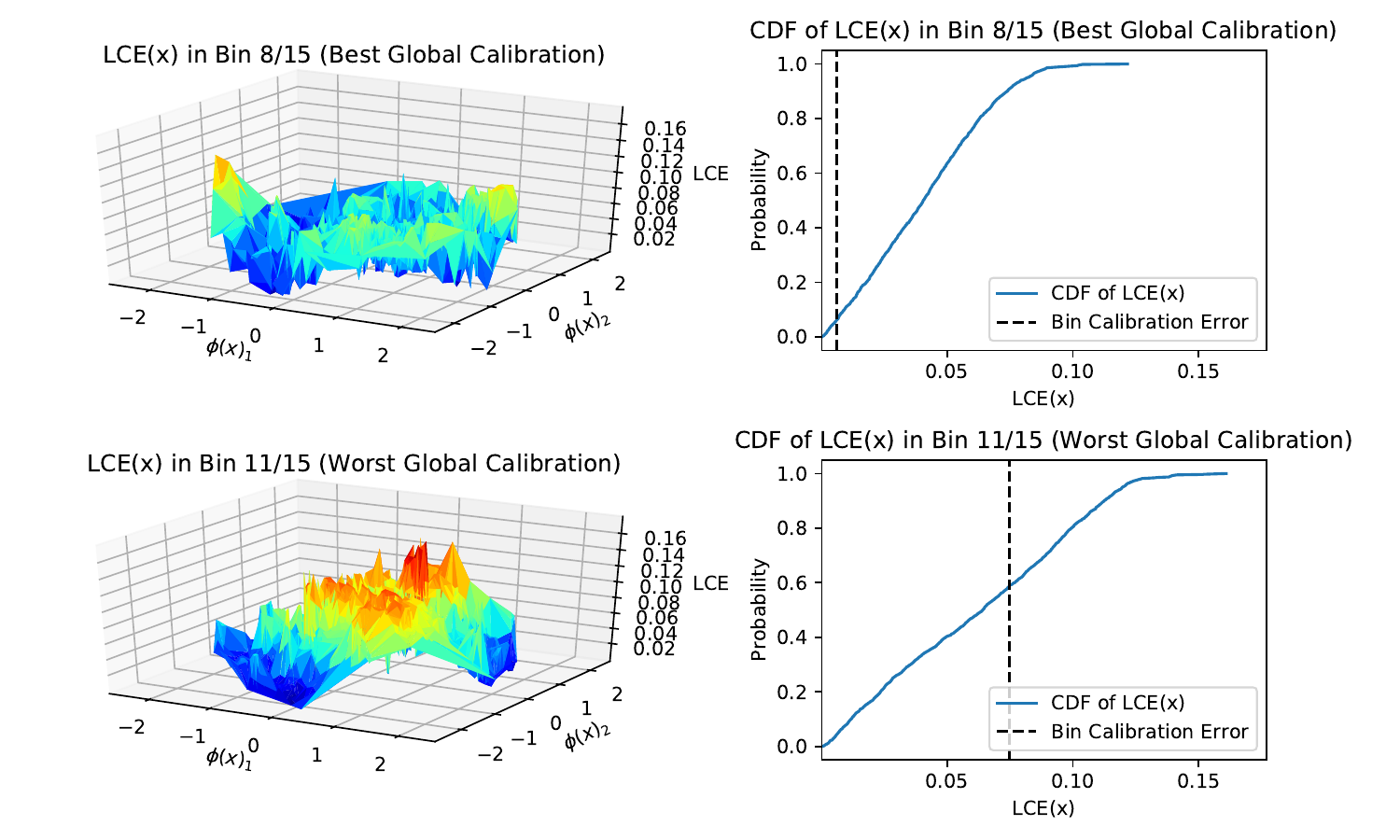}
    \caption{We visualize $\LCE_{0.2}(x; f, \hp)$ for a ResNet-50 classifier $(f, \hp)$ pre-trained on ImageNet, for every image $x$ in the ImageNet validation set. We focus on the bins with the best and worst global calibration errors.}
    \label{fig:landscape}
\end{figure*}

To provide more intuition for the LCE, we will now visualize some examples of the LCE metric over a 2-D feature embedding.  
We consider a ResNet-50 model pre-trained on ImageNet as our classifier $(f, \hp)$, and pre-trained Inception-v3 features as a feature map $\phi_1: \mathcal{X} \to \mathbb{R}^{2048}$.  $\phi_2: \mathbb{R}^{2048} \to \mathbb{R}^{2}$ then reduces the 2048-D feature vectors with t-SNE to two dimensions for ease of visualization in the LCE landscapes, so our overall representation function is $\phi = \phi_2(\phi_1(x))$. Figure \ref{fig:landscape} visualizes the landscape of $\LCE_{0.2}(x; f, \hp)$ as a function of $\phi(x)$ for the entire ImageNet validation set, as well as the marginal CDF of $\LCE_{0.2}(x, f, \hp)$. We show these visualizations for the two confidence bins with the best and worst {\em global} calibration.

In the bin with the best global calibration, Figure \ref{fig:landscape} (top) shows that the landscape of the LCE is highly non-uniform, and the CDF of the LCE lies almost entirely to the right of the bin's average calibration error. Numerically, the bin's average calibration error is $0.0061$, while its average LCE is $0.0411$. 
This implies that the regions where the model is underconfident and overconfident are spatially clustered within the bin. Because global calibration metrics solely consider the average accuracy and average confidence within a bin, confidence predictions that are too high and too low are averaged out to obtain a low overall error value; they fail to capture this {\em localized} miscalibration. 

In the bin with the worst global calibration, Figure \ref{fig:landscape} (bottom) clearly shows that the LCE still has high variance, even though the average calibration error of the bin ($0.0746$) is much closer to its average LCE ($0.0629$). The landscape also shows a clear cluster with higher LCE, indicating that the miscalibrated samples are similar in the feature space.

\section{LCE Recalibration}

\begin{figure*}[t]
    \centering
    \includegraphics[width=0.87\textwidth, trim={0cm 0cm 0cm 0cm}]{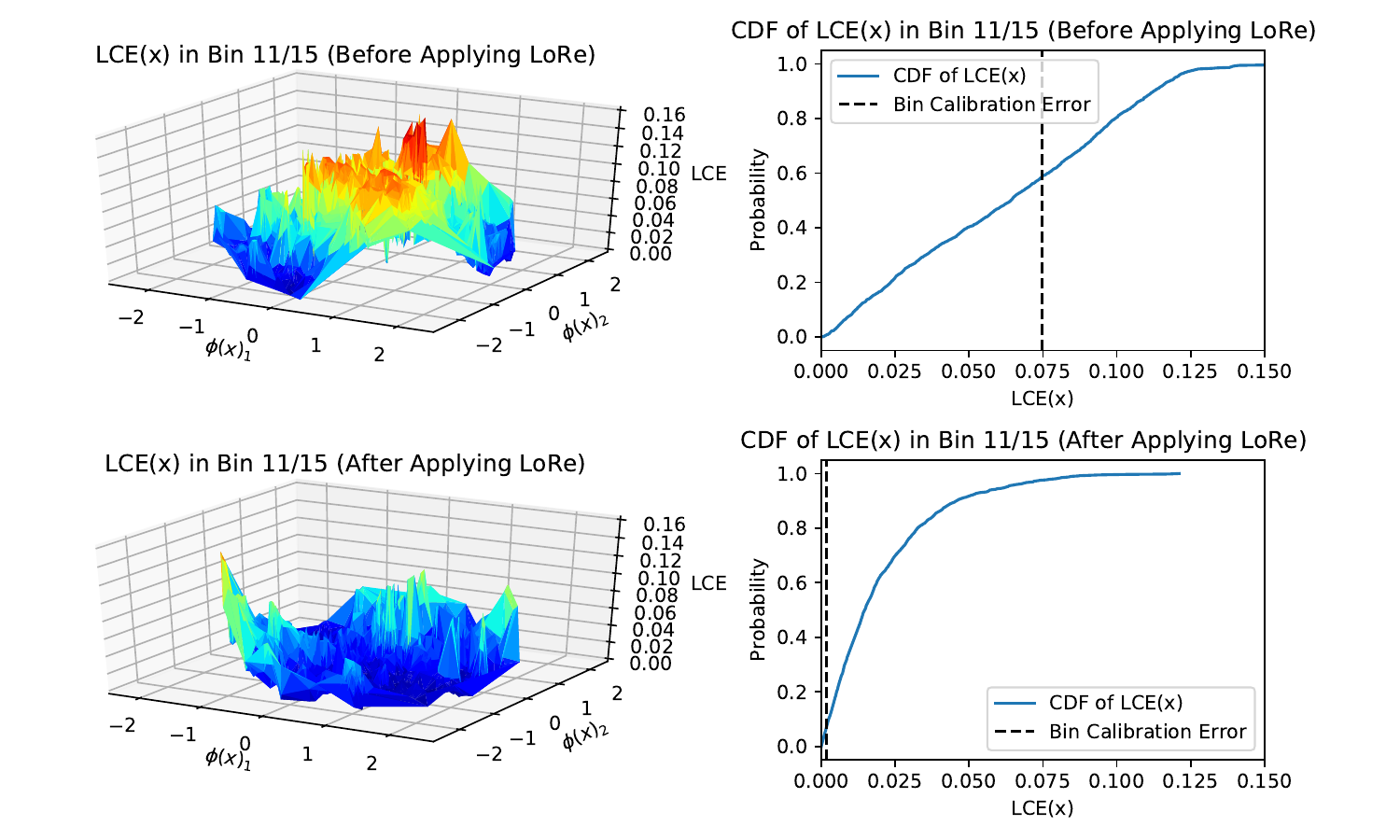}
    \caption{We visualize $\LCE_{0.2}(x; f, \hp)$ for a ResNet-50 classifier $(f, \hp)$ pre-trained on ImageNet, for every image $x$ in the ImageNet validation set for bin 11 (the bin with the worst global calibration before applying \method{}). \method{} makes the landscape flatter and lower.}
    \label{fig:landscape_bin11}
\end{figure*}

In this section, we introduce \textbf{lo}cal \textbf{re}calibration (\method), a non-parametric recalibration method that adjusts a model's output confidences to achieve better local calibration. Our method improves the LCE more than existing recalibration methods, and using our method improves performance on both downstream fairness tasks and downstream decision-making tasks. Specifically, we can leverage the kernel similarity to achieve strong calibration for all sensitive subgroups of a population, without knowing those groups a priori. As long as the feature space is semantically meaningful, \method{} provides utility for downstream tasks without needing subgroup labels for the samples.
If the subgroups {\em are} known, one can recover standard group-wise recalibration methods (and metrics) by using the improper kernel $k(x, x') = \mathds{1}[x, x' \text{ in same group}]$.

The idea behind our method is simple: we can compute the kernel-weighted accuracy for each point $x$ of all the points that are in the same confidence bin as $x$, and then reset the confidence of $x$ to this kernel-weighted accuracy value. Note that using the kernel function to compute this value is intuitively like taking a weighted average of the accuracy of the points in the local neighborhood of $x$. Thus, \method{} can be considered a local analogue to histogram binning. 

More formally, given a trained classifier $(f, \hp)$, a recalibration dataset $\mathcal{D} = ((x_1, y_1), \ldots, (x_N, y_N))$, and a fixed point $x \in \Xc$, let $\beta(x) = \{i : \hp(x_i) \in B(\hp(x)) \}$ be the set of indices of the points in $\mathcal{D}$ occupying the same confidence bin as $x$. Then, we compute the recalibrated confidence as
\begin{equation}
    \hp'(x) = \frac{\sum_{i \in \beta(x)} k_{\gamma}(x, x_i) \mathds{1}[f(x_i) = y_i]}{\sum_{i \in \beta(x)} k_{\gamma}(x, x_i)}.
    \label{eq:lore}
\end{equation}

Equation~\ref{eq:lore} represents the kernel-weighted average accuracy of all points in the same confidence bin as $x$. 
In the limit as the kernel bandwidth $\gamma \to \infty$, $\hp'(x) \to \sum_{i \in \beta(x)} \mathds{1}[f(x_i) = y_i] / \abs{\beta(x)}$ recovers histogram binning. As $\gamma \to 0$, $\hp'(x) \to \mathds{1}[f(x_{i^*}) = y_{i^*}]$, where $i^* = \arg \min_{i \in \beta(x)} k_{\gamma}(x, x_i)$, thus recovering a nearest-neighbor method. For intermediate $\gamma$, 
our method interpolates between the two extremes. 
Throughout this work, we used $\gamma = 0.2$ for \method{} with tSNE and $\gamma = 0.4$ for \method{} with PCA throughout this work, since these represent intermediate points between the limiting behaviors of the LCE (e.g., see Fig.~\ref{fig:mlce_limits}).

In Figure~\ref{fig:landscape_bin11}, we visualize the landscape of $\LCE_{0.2}(x; f, \hp)$ for the ImageNet validation set both before and after applying \method{}, for bin 11 (the bin with the worst global calibration before applying \method{}). Note that the LCE landscape before applying \method{} has an area that is raised relative to the rest of the landscape, indicating systematic miscalibration. However, after applying \method{}, the landscape is both flatter and lower, indicating improved global and local calibration, and indeed the CDF plot shows a sharper rise than before. Thus, \method{} works as desired.  

\section{Experiments}

\label{sec:experiments}

\begin{figure*}[t]
    \fontsize{8.5}{10}\selectfont
    \centering
    \begin{tabular}{l|c|c|c|c}
        \toprule
        Recalibration method & Setting 1 & Setting 2 & Setting 3 & Setting 4 \\
        \midrule
        No recalibration
        & 0.588 ± 0.107 
        & 0.407 ± 0.087 
        & 0.446 ± 0.083
        & 0.480 ± 0.122
        \\

        Temperature scaling
        & 0.521 ± 0.092 
        & 0.532 ± 0.089  
        & 0.441 ± 0.079 
        & 0.403 ± 0.108
        \\

        Histogram binning
        & 0.515 ± 0.081 
        & 0.218 ± 0.056 
        & 0.268 ± 0.067
        & 0.368 ± 0.108
        \\

        Isotonic regression
        & 0.596 ± 0.063 
        & 0.615 ± 0.100 
        & 0.716 ± 0.082 
        & 0.425 ± 0.047
        \\

        MMCE optimization
        & 0.526 ± 0.172 
        & 0.429 ± 0.079 
        & 0.475 ± 0.079 
        & 0.411 ± 0.088
        \\

        \midrule
        \added{Group temp.\ scaling}
        & \added{0.423 ± 0.066}
        & \added{0.673 ± 0.075}
        & \added{0.329 ± 0.108}
        & 0.411 ± 0.110
        \\
        \added{Group hist.\ binning}
        & \added{0.542 ± 0.083}
        & \added{0.260 ± 0.053}
        & \added{0.352 ± 0.068}
        & 0.414 ± 0.090
        \\

        \midrule
        \method{} (tSNE) (ours)
        & \textbf{0.351 ± 0.084}
        & \textbf{0.165 ± 0.055}
        & 0.235 ± 0.063
        & \textbf{0.215 ± 0.037}
        \\

        \added{\method{} (PCA) (ours)}
        & \added{0.392 ± 0.071}
        & \added{0.167 ± 0.013}
        & \added{\textbf{0.154 ± 0.082}}
        & 0.300 ± 0.065
        \\
        \bottomrule
    \end{tabular}
    \captionof{table}{
    Performance on downstream fairness, as measured by maximum group-wise MCE (lower is better).
    Experimental settings as described in Section \ref{sec:fairness}. Mean and standard deviations are computed over 60 random seeds for settings 1 and 4, and 20 for settings 2 and 3. Best results are \textbf{bold}.
    } 
    \label{tab:fairness}
\end{figure*}

In this section, we show empirically that \method{} substantially improves LCE values, and that these lower LCE values lead to better performance on downstream fairness and decision-making tasks. In particular, we evaluate the local calibration through the MLCE, because we are interested in understanding a model's {\em worst-case} local miscalibration. On each task, we compare the performance of \method{} to no recalibration (\added{`Original'}), temperature scaling (\added{`TS'}) \citep{guo17-ts}, histogram binning (\added{`HB'}) \citep{zadrozny01-hb}, isotonic regression (\added{`IR'}) \citep{zadrozny02-isotonic}, and direct MMCE optimization (\added{`MMCE'}) \citep{kumar18a-mmce}, all strong {\em global} recalibration methods. 

We first run extensive experiments on four datasets to demonstrate that \method{} outperforms all baselines and achieves the lowest MLCE over a wide range of $\gamma$ values. We then evaluate the performance of our method on a fairness task, where it is important that a model is well-calibrated for all sensitive subgroups of a given population, and we demonstrate that it achieves the lowest group-wise MCE. Notably, we find that the MLCE is well-correlated with the group-wise MCE across all experimental settings, and thus achieving low MLCE is a good indicator that a model has good group-wise calibration. Finally, we compare our method against the baselines on a cost-sensitive decision-making task, where there is a low cost for a prediction of ``unsure'' but a high cost for an incorrect prediction, and show that our method achieves the lowest cost.

\begin{figure}
    \centering
    \includegraphics[width=0.8\linewidth,trim={0cm 0cm 0cm 0cm}]{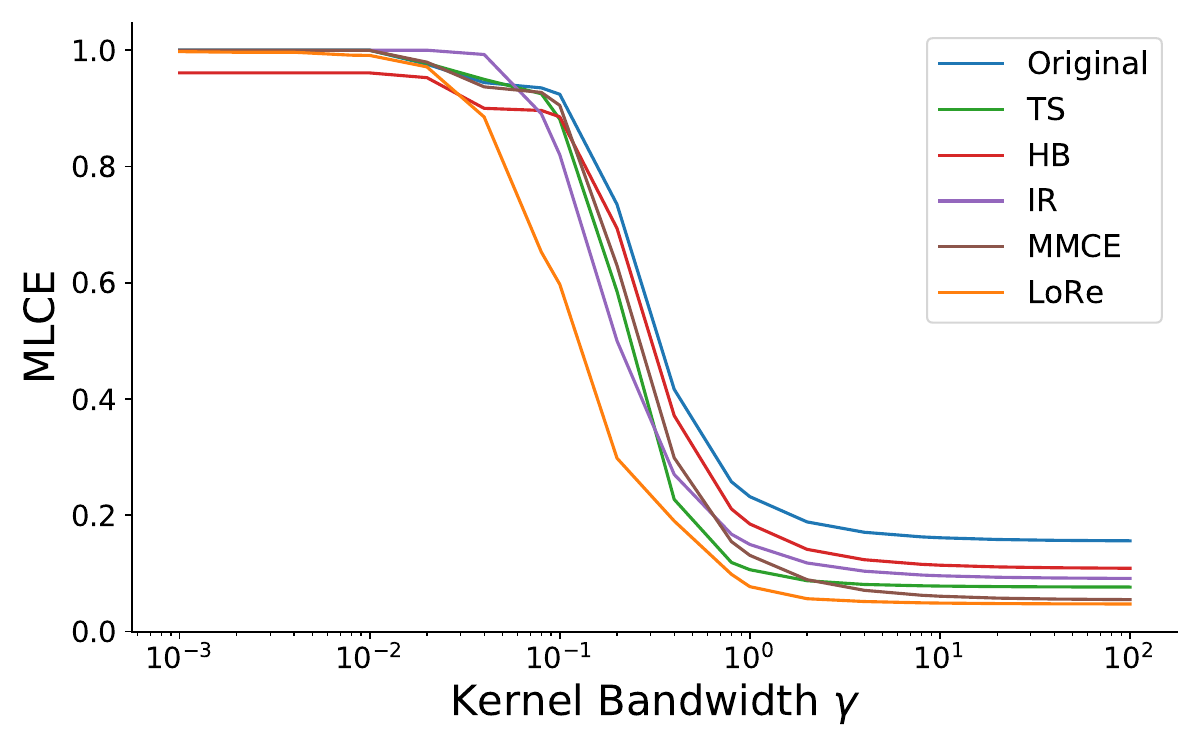}
    \caption{MLCE vs.\ kernel bandwidth $\gamma$ for ImageNet. \method{} (with t-SNE and $\gamma = 0.2$) achieves the lowest MLCE for a wide range of $\gamma$. This suggests that \method{} leads to lower LCE values across the whole dataset.}
    \label{fig:recalibration_mlce}
\end{figure}

\subsection{Datasets}
\label{sec:datasets}

\textbf{ImageNet dataset \citep{deng2009imagenet}:} A large-scale dataset of natural scene images with 1000 classes; over 1.3 million images total. The training/validation/test split is 1.3mil\,/\,25,000\,/\,25,000.

\textbf{UCI Communities and Crime dataset \citep{dua-uci}:} This tabular dataset contains attributes about American neighborhoods (e.g., race, age, employment, housing, etc.). The task is to predict the neighborhood's violent crime rate. The training/validation/test split is 1494\,/\,500\,/\,500. We randomize this split over multiple trials.

\textbf{CelebA dataset \citep{liu2015faceattributes}:} A large-scale dataset of face images with 40 attribute annotations (e.g., glasses, hair color, etc.); 202,599 images total. The training/validation/test split is 162,770\,/\,19,867\,/\,19,962. 

\textbf{COMPAS Criminal Recidivism dataset \citep{compas2016}:} This tabular dataset reports American individuals' demographic information and criminal history. The task is to predict whether a given offender will commit another violent crime within 2 years. The training/validation/test split is 2,020\,/\,1,000\,/\,1,000. We randomize this split over multiple trials.

\subsection{Recalibration Performance}

\method{} substantially improves the LCE values. In Figure~\ref{fig:recalibration_mlce}, we plot the MLCE as a function of $\gamma$.  We can see that our method outperforms all baselines (strong global calibration methods) across a wide range of $\gamma$ values on ImageNet. This is true despite the fact that we only implement \method{} for a single $\gamma$. Appendix \ref{appendix:more_results} provides similar results on the Communities \& Crime, CelebA, CIFAR-10, and CIFAR-100 datasets. \added{Note that \method{} works well regardless of the feature map and the dimensionality reduction method (see Section \ref{sec:fairness} for results with both t-SNE and PCA).
Although the results shown in this section use Inception-v3 features, we show similar results in Appendix \ref{appendix:more_results} with AlexNet \citep{AlexNet}, DenseNet121 \citep{DenseNet}, and ResNet101 \citep{ResNet} features.} 

Recall that as $\gamma$ gets large, the MLCE recovers the MCE; because \method{} does well even at large $\gamma$, our method also works well at minimizing global calibration errors. The fact that \method{} lowers the worst-case LCE suggests that it leads to lower LCE values across the entire dataset.

\begin{figure}
    \fontsize{8.5}{10}\selectfont
    \centering
    \begin{tabular}{l|c|c|c|c}
        \toprule
        &  Setting 1    & Setting 2 & Setting 3 & Setting 4 \\
        \midrule
        ECE                 & 0.102 & -0.061    & -0.195    &  0.012\\
        MCE                 & 0.233 & 0.439     & 0.281     & 0.387\\
        NLL                 & 0.542 & 0.045     & -0.287    & 0.051\\
        Brier               & 0.101 & 0.144     & -0.280    & 0.024\\
        \midrule
        MLCE (tSNE) & \textbf{0.642} & \textbf{0.801}     & 0.591 & 0.566    \\
        \added{MLCE (PCA)} & \added{0.639} & \added{0.659}     & \added{\textbf{0.778}}  & \textbf{0.603}   \\
        \bottomrule
    \end{tabular}
    \captionof{table}{Pearson correlation between max group-wise MCE and other calibration metrics (higher is better). Experimental settings as described in Section \ref{sec:fairness}. Best results in \textbf{bold}. We use $\gamma = 0.2$ for MLCE (tSNE) and $\gamma = 0.4$ for MLCE (PCA). MLCE is better-correlated with the max group-wise MCE than any of the global metrics in all settings.}
    \label{tab:correlation}
\end{figure}

\subsection{Downstream Fairness Performance}

\label{sec:fairness}
\begin{figure*}
    \fontsize{8}{10}\selectfont
    \setlength{\tabcolsep}{3pt}
    \centering
    \begin{tabular}{l|ccc|ccc|ccc|ccc}
        \toprule
        Recalibration method &
        \multicolumn{3}{c|}{Setting 1} & \multicolumn{3}{c|}{Setting 2} & \multicolumn{3}{c|}{Setting 3} & \multicolumn{3}{c}{Setting 4} \\
        & ECE(\%) & NLL & Brier & ECE(\%) & NLL & Brier & ECE(\%) & NLL & Brier & ECE(\%) & NLL & Brier\\
        \midrule
        No recalibration
        & $15.1_{2.7}$ & $.96_{.25}$ & $.17_{.02}$
        & $\mathbf{1.8_{0.3}}$ & $\mathbf{.617_{.004}}$ & $.641_{.004}$ 
        & $1.1_{0.3}$ & $.782_{.006}$ & $.571_{.006}$
        & $3.6_{1.4}$ & $.408_{.023}$ & $\mathbf{.124_{.008}}$
        \\

        Temperature scaling
        & $4.9_{1.7}$ & $.43_{.03}$ & $.14_{.01}$
        & $2.0_{0.3}$ & $.619_{.004}$ & $.622_{.003}$ 
        & $\mathbf{1.0_{0.2}}$ & $\mathbf{.781_{.006}}$ & $.569_{.002}$
        & $2.9_{1.2}$ & $\mathbf{.405_{.021}}$ & $\mathbf{.124_{.007}}$
        \\

        Histogram binning
        & $\mathbf{3.3_{1.1}}$ & $.48_{.03}$ & $.15_{0.1}$
        & $2.5_{0.2}$ & $.619_{.004}$ & $.614_{.003}$ 
        & $2.5_{0.4}$ & $.788_{.006}$ & $.552_{.002}$
        & $2.7_{1.1}$ & $.414_{.023}$ & $.126_{.008}$
        \\

        Isotonic regression
        & $30.6_{2.3}$ & $.79_{.05}$ & $.30_{.02}$
        & $2.6_{0.2}$ & $.618_{.004}$ & $.615_{.003}$ 
        & $2.4_{0.2}$ & $.785_{.006}$ & $.553_{.002}$
        & $34.4_{1.7}$ & $.707_{.021}$ & $.253_{.004}$
        \\

        MMCE optimization
        & $4.4_{1.3}$ & $.43_{.03}$ & $.14_{.01}$
        & $3.8_{0.7}$ & $.646_{.014}$ & $.679_{.012}$ 
        & $5.4_{0.8}$ & $.808_{.009}$ & $.619_{.009}$
        & $3.4_{1.9}$ & $.415_{.039}$ & $.125_{.009}$
        \\

        \midrule
        \method{} (tSNE) (ours)
        & $3.5_{1.1}$ & $\mathbf{.42_{.02}}$ & $\mathbf{.13_{.01}}$
        & $2.8_{0.2}$ & $.623_{.004}$ & $.613_{.003}$ 
        & $2.6_{0.3}$ & $.792_{.006}$ & $.551_{.002}$
        & $2.7_{1.1}$ & $.410_{.022}$ & $.125_{.007}$
        \\
        
        \added{\method{} (PCA) (ours)}
        & \added{$4.5_{1.4}$} & \added{$.44_{.02}$} & \added{$.14_{.01}$}
        & \added{$3.1_{0.2}$} & \added{$.628_{.004}$} & \added{$\mathbf{.606_{.003}}$}
        & \added{$2.8_{0.4}$} & \added{$.792_{.007}$} & \added{$\mathbf{.538_{.002}}$}
        & $\mathbf{2.5_{1.0}}$ & $.410_{.021}$ & $.126_{.008}$
        \\
        \bottomrule
    \end{tabular}
    \captionof{table}{
    Performance on global calibration metrics, formatted as $\text{mean}_{\text{sd}}$. Lower is better.
    Experimental settings as described in Section \ref{sec:fairness}.
    Best results are \textbf{bold}. Across all settings, \method{} generally achieves a global calibration error that is comparable to the baselines.
    } 
    \label{tab:global}
\end{figure*}

\paragraph{Experimental Setup}
In many fairness-related applications, it is important to show that a model is well-calibrated for all sensitive subgroups of a given population. For example, when predicting the crime rate of a neighborhood, a model should not be considered well-calibrated if it consistently underestimates the crime rate for neighborhoods of one demographic, while overestimating the crime rate for neighborhoods of a different demographic.
Therefore, in this section, we examine the worst-case group-wise miscalibration of a classifier, as measured by the maximum group-wise MCE when evaluated only on sensitive sub-groups. 
We consider the following experimental settings: 
\begin{enumerate}[noitemsep,topsep=0pt]
    \item UCI Communities and Crime: Predict whether a neighborhood's crime rate is higher than the median; group neighborhoods by their plurality race (White, Black, Asian, Indian, Hispanic). 60 random seeds for model training.
    \item CelebA: Predict a person's hair color (bald, black, blond, brown, gray, other); group people by hair type (bald, receding hairline, bangs, straight, wavy, other). 20 random seeds for model training.
    \item CelebA: Predict a person's hair type; group people by their hair color; inverse of Setting 2. 20 random seeds for model training.
    \item COMPAS Criminal Recidivism: Predict whether an individual commits another violent crime within 2 years; group individuals based on race. 60 random seeds for model training.
\end{enumerate}

For each task, we train a classifier (see Appendix \ref{appendix:details} for full details) and recalibrate its output confidences using each of the recalibration methods.

\paragraph{Results}
Table \ref{tab:fairness} reports the maximum group-wise MCE for each of the recalibration methods on each of the three tasks. \method{} outperforms the other baselines, achieving an average 50\% reduction over no recalibration and an average 24\% improvement over the next best global recalibration method. (Figures \ref{fig:setting1_mlce}, \ref{fig:setting2_mlce}, \ref{fig:setting3_mlce}, and \ref{fig:setting4_mlce} in Appendix \ref{appendix:more_results} show that \method{} is also the best method of lowering the MLCE over a wide range of $\gamma$). \added{Notably, \method{} is robust to the feature map used (tSNE vs.\ PCA). It even outperforms global methods applied to each individual group, implying that correcting local calibration errors is a robust way to improve group calibration that generalizes better than naive alternatives.}

Moreover, Table \ref{tab:correlation} shows that the maximum group-wise MCE is well-correlated with the MLCE, and it is in fact {\em much} better correlated with MLCE than global calibration metrics.
Taken together, our results indicate that lowering the LCE has positive implications in fairness settings that cannot be achieved by simply lowering global metrics like the ECE.
For reference, we also include the performance of all methods on various global calibration metrics in Table~\ref{tab:global}, which shows that \method{} is able to improve worst-case group-wise calibration without meaningfully sacrificing (and in some cases improving) average-case global calibration. 

\subsection{Downstream Decision-Making}

\paragraph{Experimental Setup}
Machine learning predictions are often used to make decisions, and in many situations an agent must select a best action in expectation. As an example, suppose there is a low cost $u$ associated with returning ``unsure'' and a high cost $w$ associated with returning an incorrect classification (e.g., in situations such as autonomous driving, being unsure incurs only the small cost of calling a human operator, but making an incorrect classification incurs a high cost). An agent with good uncertainty quantification can make a more optimal decision about whether to return a classification or return ``unsure''; for a calibrated model, it would be optimal for the agent to return ``unsure'' below the confidence threshold of $1 - u / w$, and return a prediction above this threshold. 

Following this policy --- i.e., returning ``unsure'' when the confidence is below this threshold and returning a prediction when the confidence is above it, we used a ResNet-50 model to make predictions on ImageNet, and recalibrated the predictions with each of the recalibration methods. For each method, we then calculated the total reward attained under various reward ratios $w/u$, as well as various global calibration metrics.

\begin{figure}
    \centering
    \includegraphics[width=0.8\linewidth, trim={0cm 0cm 0cm 0cm},clip]{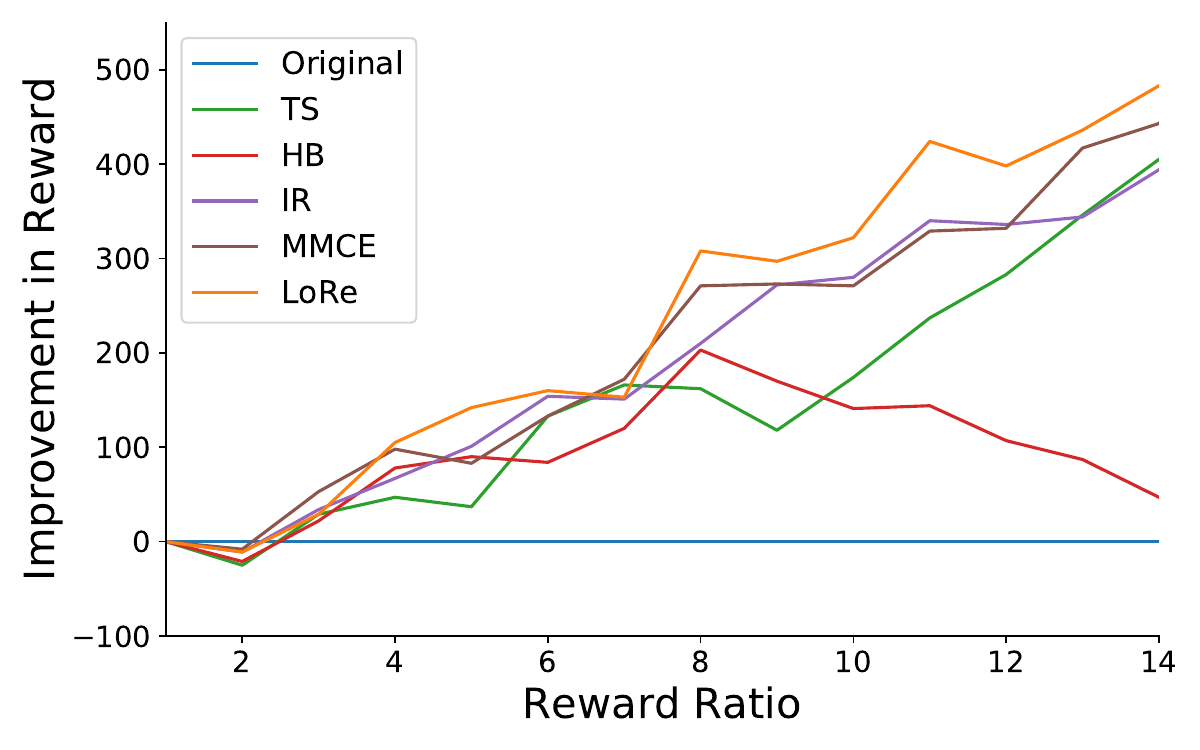}
    \caption{Reward attained vs.\ reward ratio for the ImageNet dataset (higher is better). \method{} achieves the highest rewards across a wide range of reward ratios.}
    \label{fig:reward}
\end{figure}

\paragraph{Results}
In Figure~\ref{fig:reward}, we show the improvement in the total reward over the original classifier (i.e., no recalibration) as a function of the reward ratio $w/u$ (the ratio of the cost of an incorrect classification to the cost of being unsure). Across a wide range of reward ratios, \method{} achieves the highest reward. The MLCE curves for this task are shown in Figure~\ref{fig:recalibration_mlce}; note that \method{} also achieves lower LCE values than the global recalibration methods. Table~\ref{tab:decision_making_expanded} reports several global calibration metrics; \method{} achieves strong global calibration. These results indicate that our recalibration method most effectively lowers LCE values without sacrificing (and indeed often improving) average-case global calibration, and that these lower LCE values correspond to better performance on this decision-making task. 

\begin{figure}
    \small
    \centering
	\begin{tabular}{l|ccc}
		\toprule 
		Recalibration method   
		& ECE & NLL & Brier      \\
        \midrule
        No recalibration	  
        &  0.037 & 0.959 &      40.64  \\
        Temperature scaling	 
        &  0.022 & 0.948 &     40.60 \\
        Histogram binning	  
        & 0.012 & 0.952 &    40.59  \\
        Isotonic regression   
        & 0.011 & \textbf{0.945}  &    40.59   \\
        MMCE optimization  
        & 0.061 &  0.965 &     40.67   \\
        \midrule
        \method{} (ours)        
        & \textbf{0.007} & 0.955 &  \textbf{40.58}  \\
		\bottomrule
	\end{tabular}
	\captionof{table}{
	Performance on global calibration metrics on ImageNet. Lower is better.
	Best results are \textbf{bold}. \method{} achieves strong global calibration according to all metrics.}
	\label{tab:decision_making_expanded} 
\end{figure}

\section{Conclusion}

In this paper, we introduce the local calibration error (LCE), a metric that measures calibration in a localized neighborhood around a prediction. The LCE spans the gap between fully global and fully individualized calibration error, with an effective neighborhood size that can be set with a bandwidth parameter $\gamma$. We also introduce \method{}, a recalibration method that greatly improves the local calibration. Finally, we demonstrate that achieving lower LCE values leads to better performance on downstream fairness and decision-making tasks. In future work, we hope to further explore alternative feature spaces to define similarity, since the quality of our metric depends on the quality of the feature space underpinning the notion of locality.

\bibliography{luo_333}

\clearpage
\newpage
\appendix

\onecolumn

\newtheorem{assumption}{Assumption}
\renewcommand{\theassumption}{\Alph{assumption}}
\newcommand{\set}[1]{{\left\{ #1 \right\}}}
\renewcommand{\abs}[1]{{\left| #1 \right|}}
\newcommand{\paren}[1]{{\left( #1 \right)}}
\newcommand{\brac}[1]{{\left[ #1 \right]}}
\newcommand{\mc}[1]{\mathcal{#1}}
\newcommand{\mb}[1]{\mathbf{#1}}
\renewcommand{\P}{\mathbb{P}}
\newcommand{\indic}[1]{\mathds{1}\left[#1\right]} 
\newcommand{\slce}{\mathrm{SLCE}}
\newcommand{\slcehat}{\widehat{\mathrm{SLCE}}}
\renewcommand{\norm}[1]{\left\|{#1}\right\|} 
\newcommand{\wt}{\widetilde}

\section{Model Architecture, Training, and Other Hyperparameters}
\label{appendix:details}
For ImageNet and CelebA, we compute the ECE, MCE, and LCE using 15 equal-width confidence bins. For the UCI communities and crime dataset, we use 5 equal-width bins because the dataset is much smaller (500 datapoints for recalibration). These numbers of bins represent a good tradeoff between bias and variance in estimating the relevant calibration errors. We also ran some initial experiments with equal-mass binning, but found that the results were very similar to those obtained with equal-width binning.

\subsection{ImageNet}
For all experiments with the ImageNet dataset, we used the pre-trained ResNet-50 model from the PyTorch \texttt{torchvision} package as our classifier. To calculate the LCE and apply \method{}, we used pre-trained Inception-v3 features, applying either t-SNE to reduce their dimension to 3 \added{or PCA to reduce their dimension to 50}, as a feature representation for the kernel.

\subsection{UCI Communities and Crime}
For all experiments with the UCI communities and crime dataset, we used a 3-hidden-layer dense neural network as our base classifier. Each hidden layer had a width of 100 and was followed by a Leaky ReLU activation. We applied dropout with probability 0.4 after the final hidden layer. We trained the model using the Adam optimizer with a batch size of 64 and a learning rate of $3 \times 10^{-4}$ until the validation accuracy stopped improving. All other hyperparameters were PyTorch defaults. Training was done locally on a laptop CPU. We trained 60 different models with different random seeds to perform the experiments described in Section \ref{sec:fairness} and Figure \ref{fig:setting1_mlce}.
To calculate the LCE and apply \method{}, we used the final hidden layer representation learned by our model, applying t-SNE to reduce the dimension to 2 \added{or PCA to reduce their dimension to 20}, as a feature representation for the kernel.

\subsection{CelebA}
For all experiments with the CelebA dataset, we trained a ResNet50 model and used it as our base classifier. We applied standard data augmentation to our training data (random crops \& random horizontal flips), and trained all models for 10 epochs 
using the Adam optimizer with a learning rate of $1 \times 10^{-3}$ and a batch size of 256. 
All other hyperparameters were PyTorch defaults. Training was distributed over 4 GPUs, and training a single model took about 30 minutes. For both Setting 2 and Setting 3 (described in Section~\ref{sec:fairness}), we trained 20 models with different random seeds to perform the experiments shown in Figures~\ref{fig:setting2_mlce} and ~\ref{fig:setting3_mlce}.
To calculate the LCE and apply \method{}, we used pre-trained Inception-v3 features, applying t-SNE to reduce their dimension to 2 \added{or PCA to reduce their dimension to 50}, as a feature representation for the kernel.

\subsection{COMPAS Criminal Recidivism}
For all experiments with the COMPAS criminal recidivism dataset, we used a 3-hidden-layer dense neural network as our base classifier. Each hidden layer had a width of 100 and was followed by a Leaky ReLU activation. We applied dropout with probability 0.4 after the final hidden layer. We trained the model using the Adam optimizer with a batch size of 64 and a learning rate of $3 \times 10^{-4}$ until the validation accuracy stopped improving. All other hyperparameters were PyTorch defaults. Training was done locally on a laptop CPU. We trained 60 different models with different random seeds to perform the experiments described in Section \ref{sec:fairness} and Figure \ref{fig:setting1_mlce}.
To calculate the LCE and apply \method{}, we used the final hidden layer representation learned by our model, applying t-SNE to reduce the dimension to 2 \added{or PCA to reduce their dimension to 20}, as a feature representation for the kernel.

\section{Additional Experimental Results}
\label{appendix:more_results}
In Figures \ref{fig:setting1_mlce}, \ref{fig:setting2_mlce}, \ref{fig:setting3_mlce}, and \ref{fig:setting4_mlce} we visualize the MLCE achieved by all recalibration methods for the three experimental settings evaluated in Section \ref{sec:fairness}. Figure \ref{fig:recalibration_mlce} in the main paper shows the same visualization for all methods on ImageNet. \added{In Figure \ref{fig:cifar100}, we plot the MLCE achieved by all recalibration methods for CIFAR-100, and in Figure \ref{fig:cifar10}, we do the same for CIFAR-10.} Across all settings and datasets, our method \method{} is the most effective at minimizing MLCE across a wide range of $\gamma$, even accounting for variations between runs. 

In these figures, ``Original'' represents no recalibration, ``TS'' represents temperature scaling, ``HB'' represents histogram binning, ``IR'' represents isotonic regression, ``MMCE'' represents direct MMCE optimization, and ``LoRe'' is our method. 

\added{Next, we examine the influence of the specific feature map used. In Figures \ref{fig:mlce_inceptionv3}, \ref{fig:mlce_alexnet}, \ref{fig:mlce_densenet121}, and \ref{fig:mlce_resnet101}, we plot the MLCE achieved by all recalibration methods for ImageNet using Inception-v3, AlexNet, DenseNet121, and ResNet101 features. In Figures \ref{fig:mlce_inceptionv3_lore_alexnet} and \ref{fig:mlce_densenet121_lore_alexnet}, we plot the MLCE achieved by all recalibration methods for ImageNet when the features used to calculate the MLCE are different from the features used by \method{}. For completeness, in Figures \ref{fig:elce_imagenet}, \ref{fig:elce_setting1}, \ref{fig:elce_setting2}, and \ref{fig:elce_setting3}, we also visualize the average LCE for all experimental settings. All plots show similar results: \method{} performs best over a wide range of $\gamma$.}

\begin{figure}[H]
    \centering
    \includegraphics[width=0.48\textwidth]{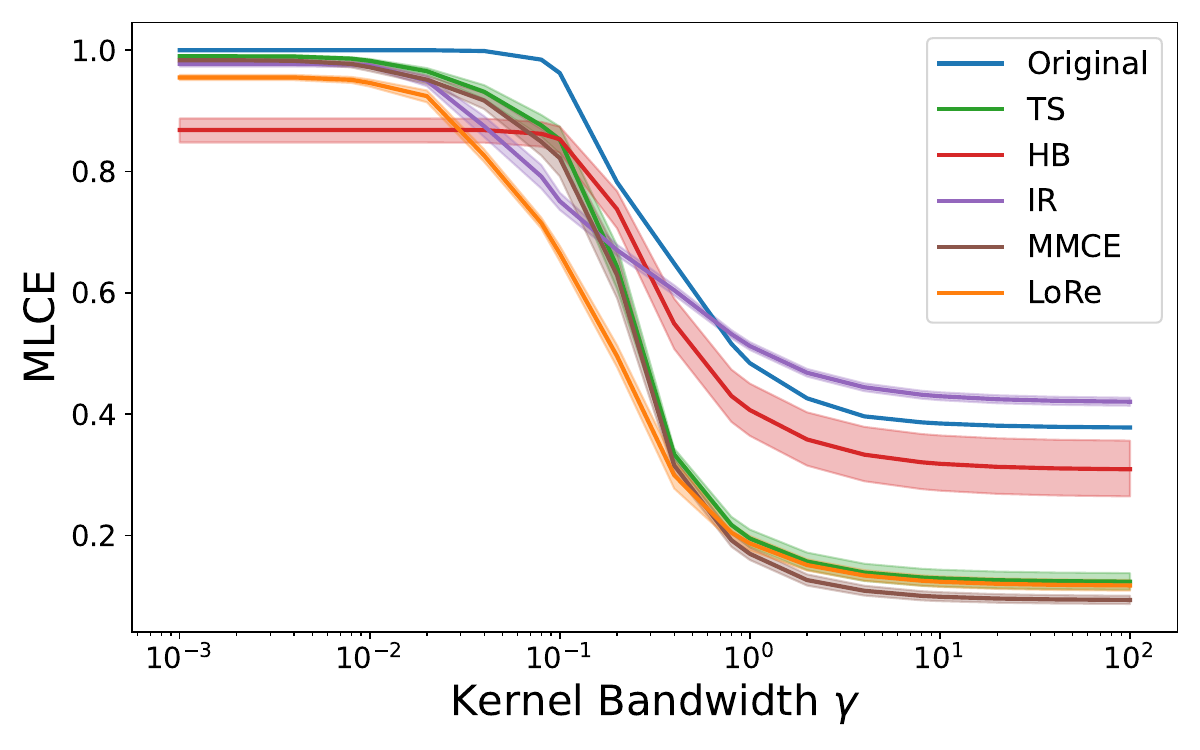}
    \hfill
    \includegraphics[width=0.48\textwidth]{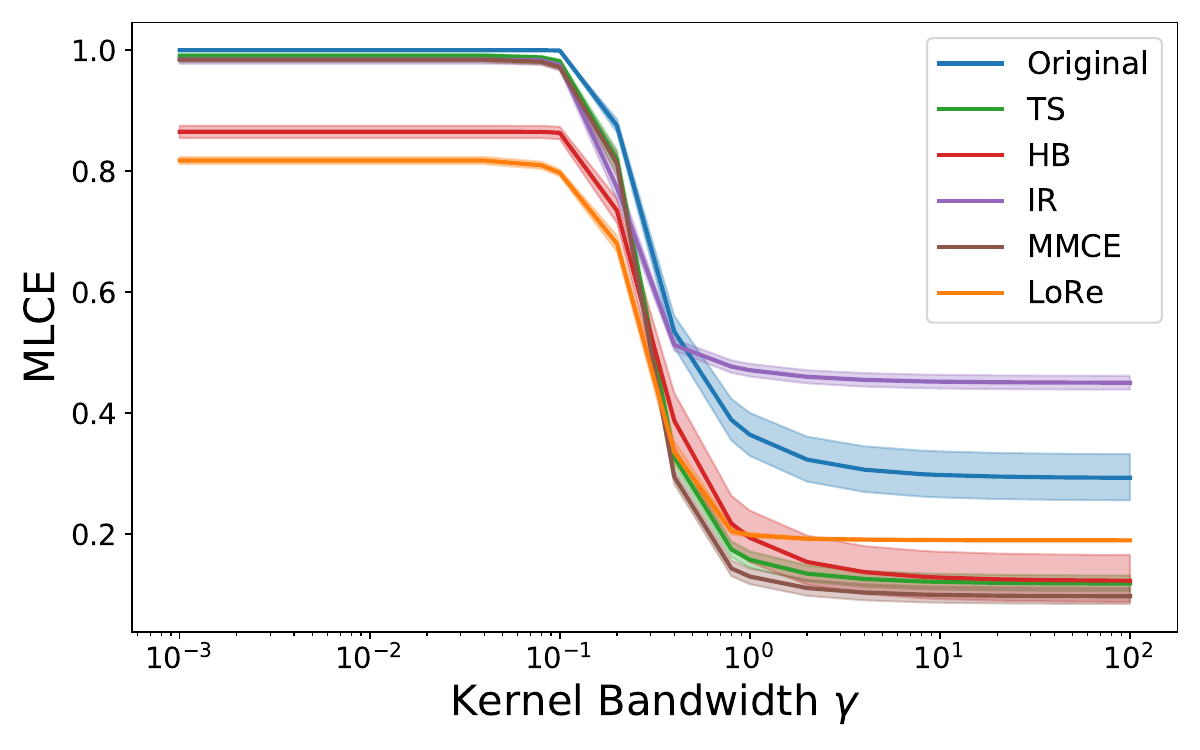}
    \caption{MLCE vs.\ kernel bandwidth $\gamma$ for all methods on task 1 of Section \ref{sec:fairness}, predicting whether a neighborhood's crime rate is higher than the median. \method{} achieves the best (or competitive) MLCE for most $\gamma$. Left: 2D t-SNE features. Right: 20D PCA features.}
    \label{fig:setting1_mlce}
\end{figure}

\begin{figure}[H]
    \centering
    \includegraphics[width=0.48\textwidth]{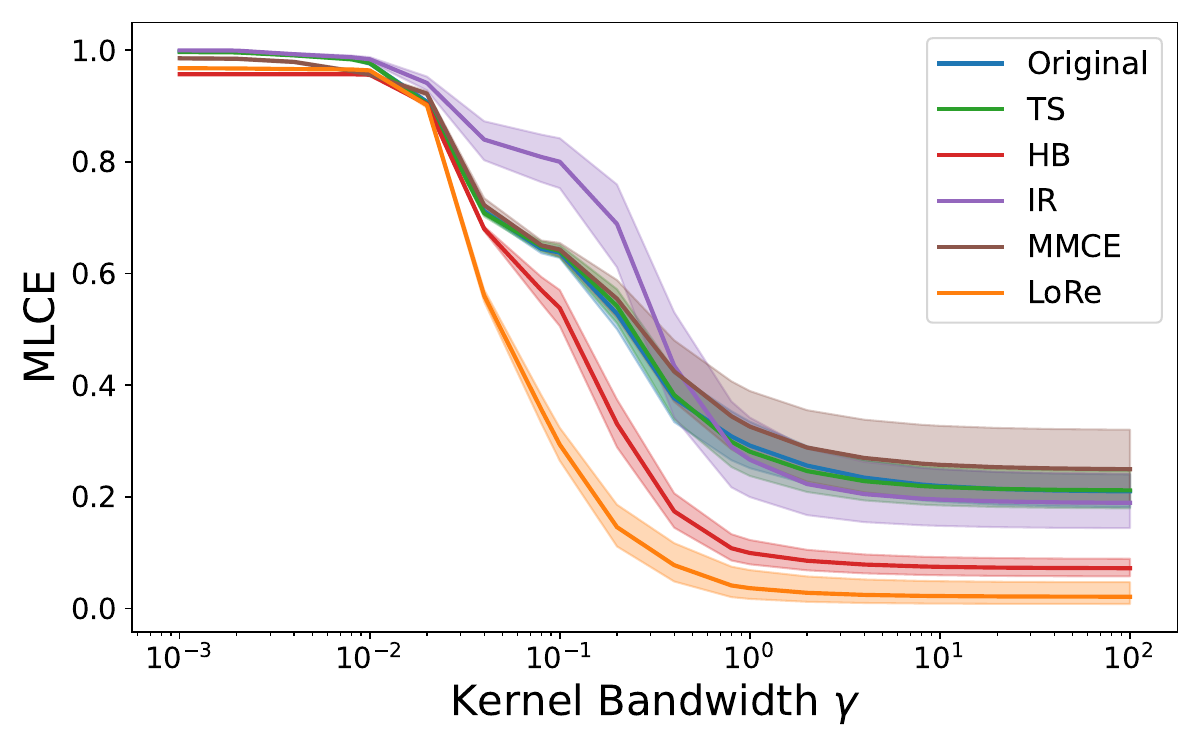}
    \hfill
    \includegraphics[width=0.48\textwidth]{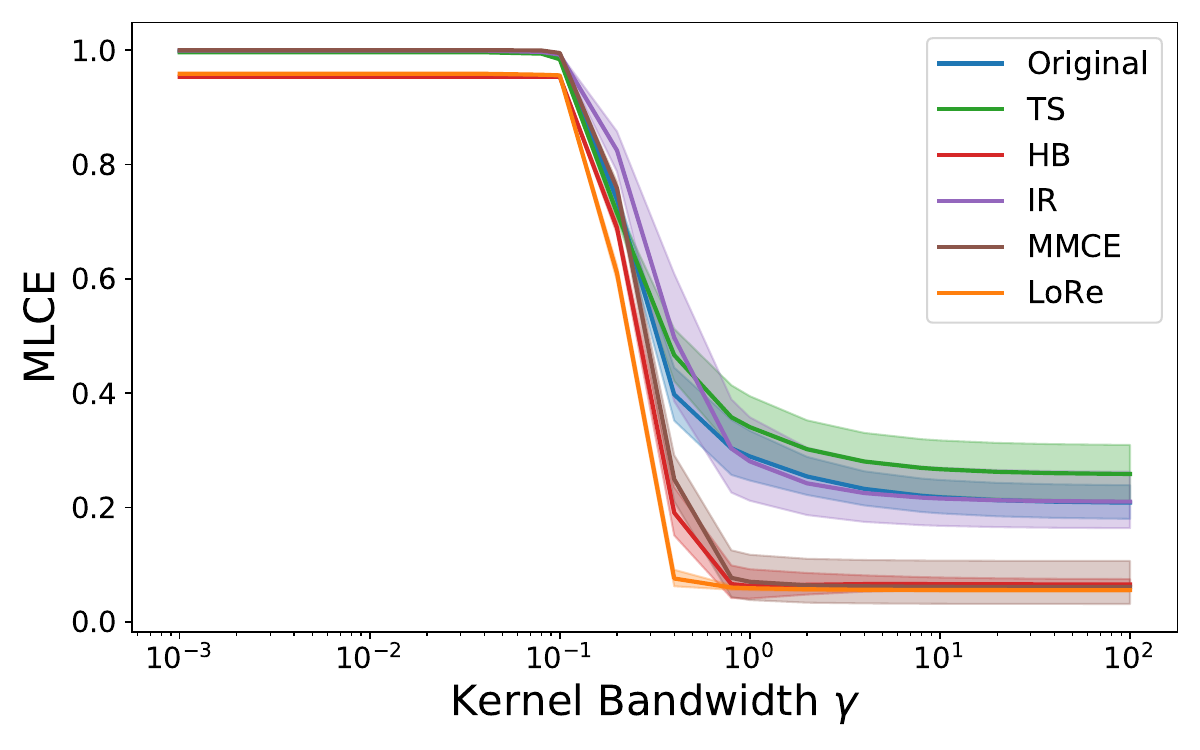}
    \captionof{figure}{MLCE vs.\ kernel bandwidth $\gamma$ for all methods on task 2 of Section \ref{sec:fairness}, predicting hair color on CelebA. \method{} achieves the best MLCE for virtually all values of $\gamma$. Left: 2D t-SNE features. Right: 50D PCA features.}
    \label{fig:setting2_mlce}
\end{figure}

\begin{figure}[h]
    \centering
    \includegraphics[width=0.48\textwidth]{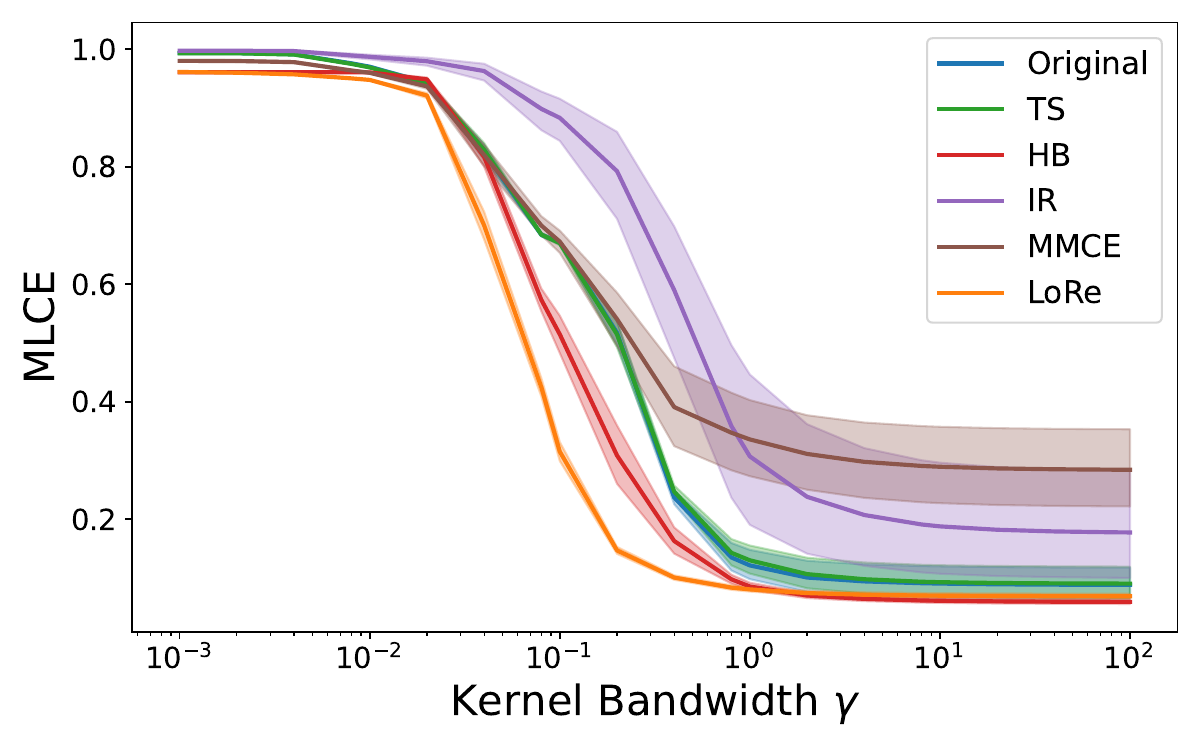}
    \hfill
    \includegraphics[width=0.48\textwidth]{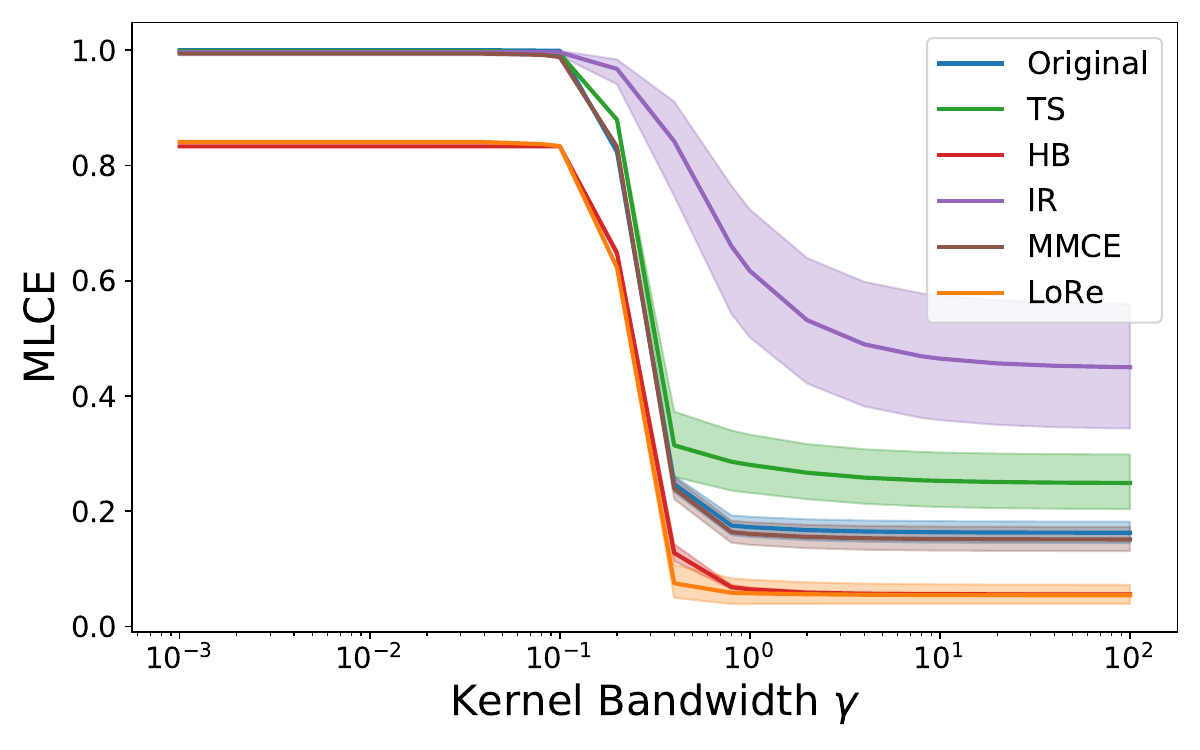}
    \captionof{figure}{MLCE vs.\ kernel bandwidth for all methods on task 3 of Section \ref{sec:fairness}, predicting hair type on CelebA. \method{} achieves the best MLCE for all $\gamma < 1$ and is tied with histogram binning for $\gamma > 1$. Left: 2D t-SNE features. Right: 50D PCA features.}
    \label{fig:setting3_mlce}
\end{figure}

\begin{figure}[H]
    \centering
    \includegraphics[width=0.48\textwidth]{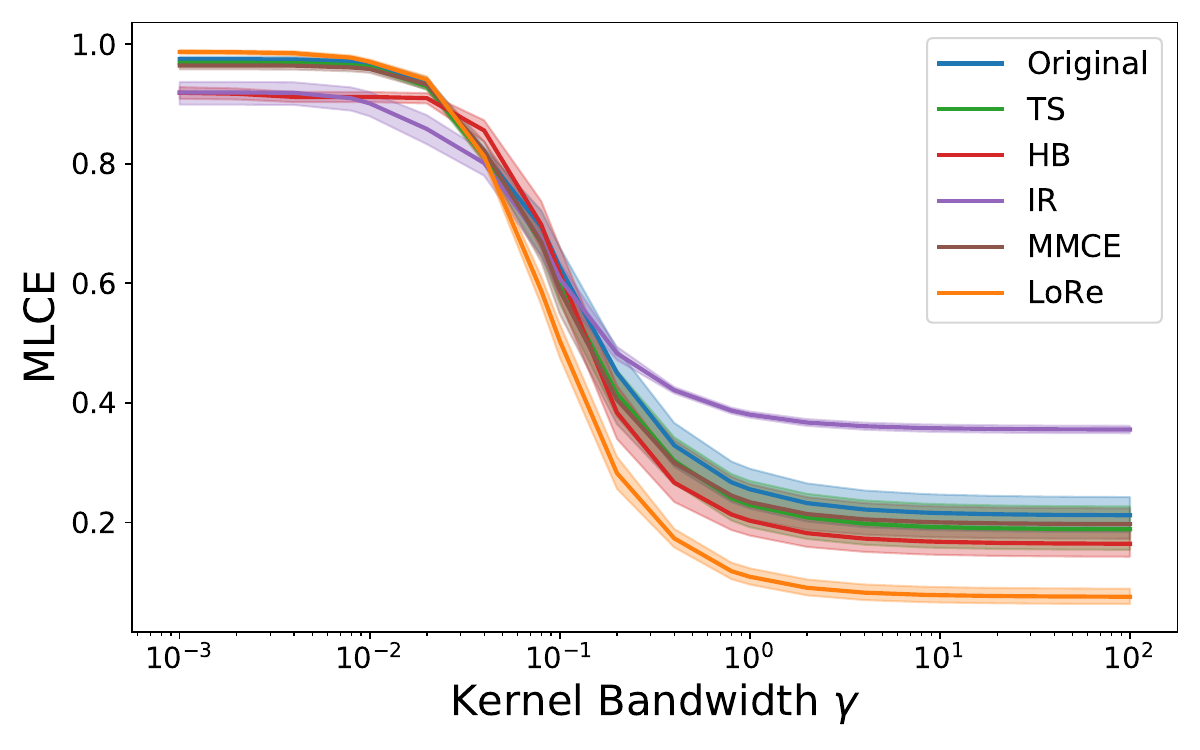}
    \hfill
    \includegraphics[width=0.48\textwidth]{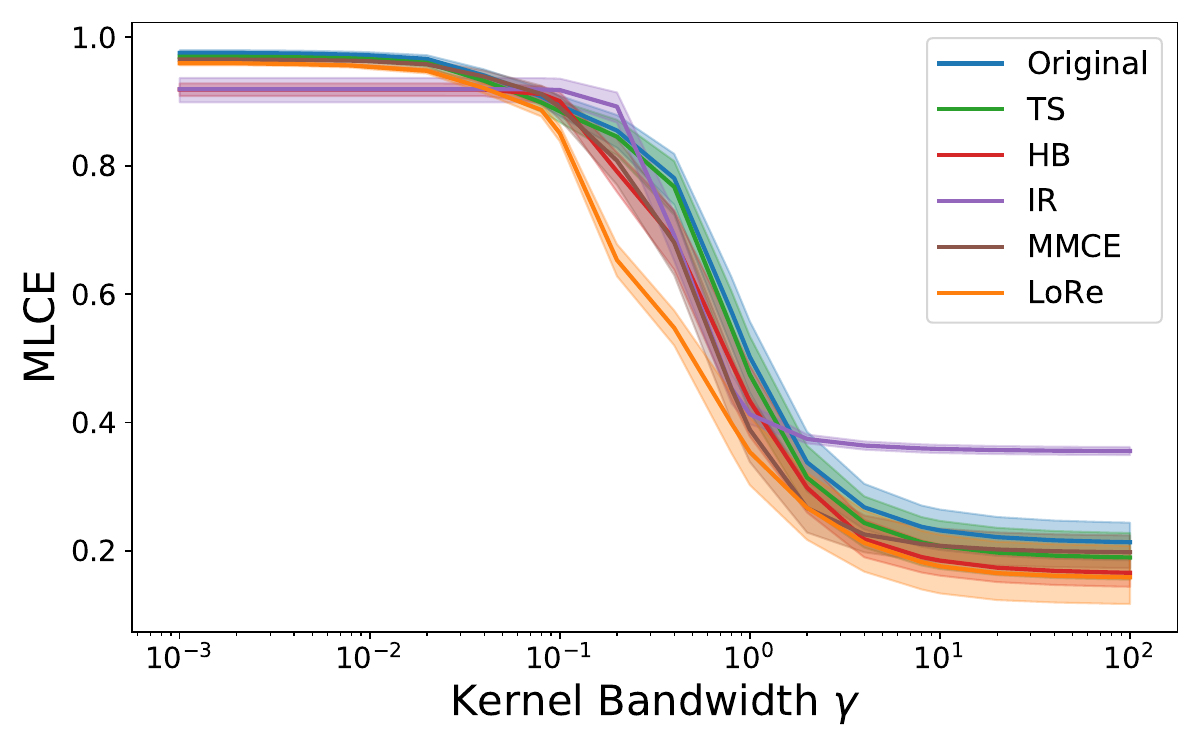}
    \captionof{figure}{MLCE vs.\ kernel bandwidth for all methods on task 4 of Section \ref{sec:fairness}, predicting criminal recidivism. \method{} achieves the best (or competitive) MLCE for most $\gamma$. Left: 2D t-SNE features. Right: 20D PCA features.}
    \label{fig:setting4_mlce}
\end{figure}

\begin{figure}[H]
    \begin{minipage}{0.48\textwidth}
        \centering
        \includegraphics[width=\linewidth,trim={0cm 0cm 0cm 0cm}]{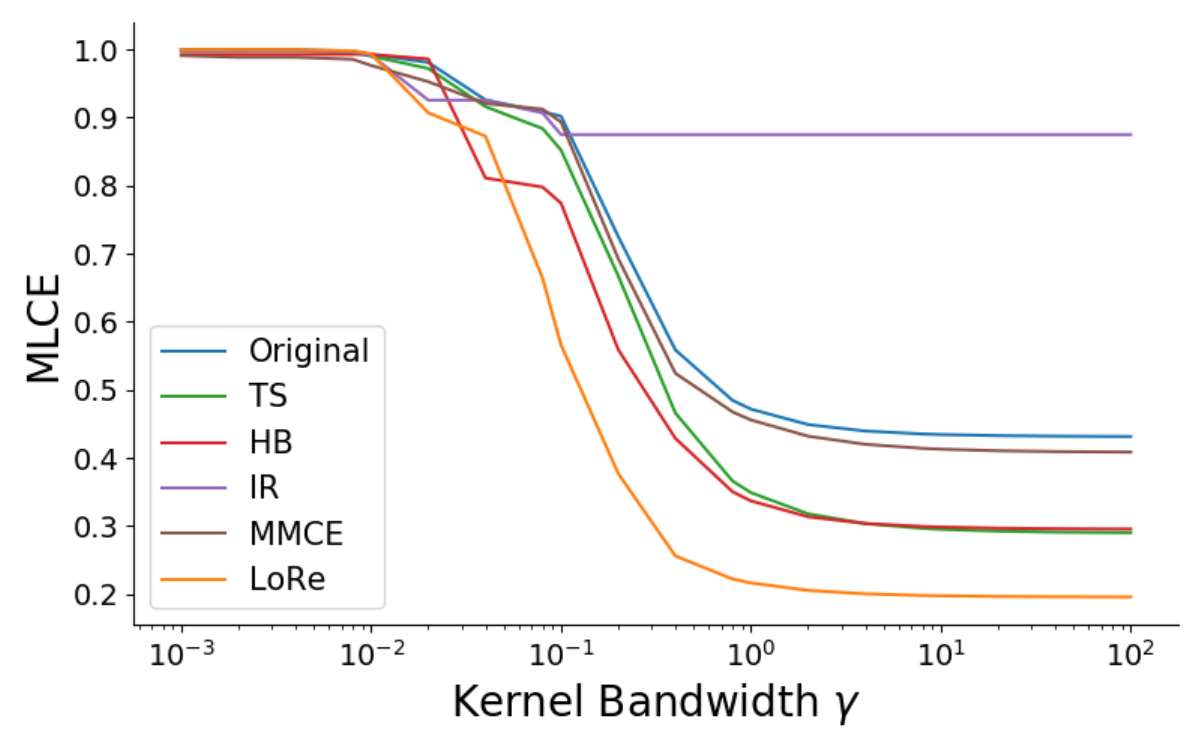}
        \caption{MLCE vs.\ kernel bandwidth $\gamma$ for all recalibration methods for CIFAR-100 (3D t-SNE features). \method{} achieves lower MLCE for most $\gamma$.}
        \label{fig:cifar100}
    \end{minipage}
    \hfill
    \begin{minipage}{0.48\textwidth}
        \centering
        \includegraphics[width=\linewidth, trim={0cm 0cm 0cm 0cm},clip]{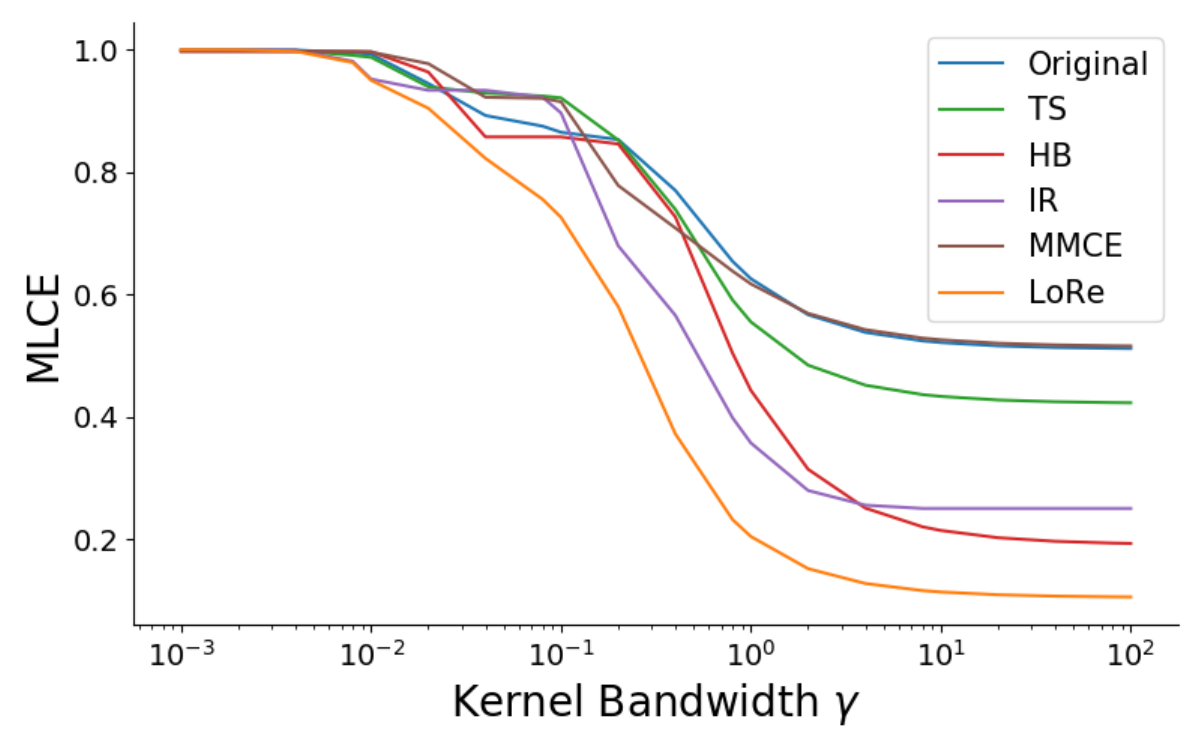}
        \caption{MLCE vs.\ kernel bandwidth $\gamma$ for all recalibration methods for CIFAR-10 (3D t-SNE features). \method{} achieves lower MLCE for most $\gamma$.}
        \label{fig:cifar10}
    \end{minipage}
\end{figure}

\begin{figure}[H]
    \begin{minipage}{0.48\textwidth}
        \centering
        \includegraphics[width=\linewidth,trim={0cm 0cm 0cm 0cm}]{figures/mlce_recalibration.pdf}
        \caption{MLCE vs.\ kernel bandwidth $\gamma$ for all recalibration methods on ImageNet using Inception-v3 features. \method{} achieves the best MLCE for most $\gamma$.}
        \label{fig:mlce_inceptionv3}
    \end{minipage}
    \hfill
    \begin{minipage}{0.48\textwidth}
        \centering
        \includegraphics[width=\linewidth, trim={0cm 0cm 0cm 0cm},clip]{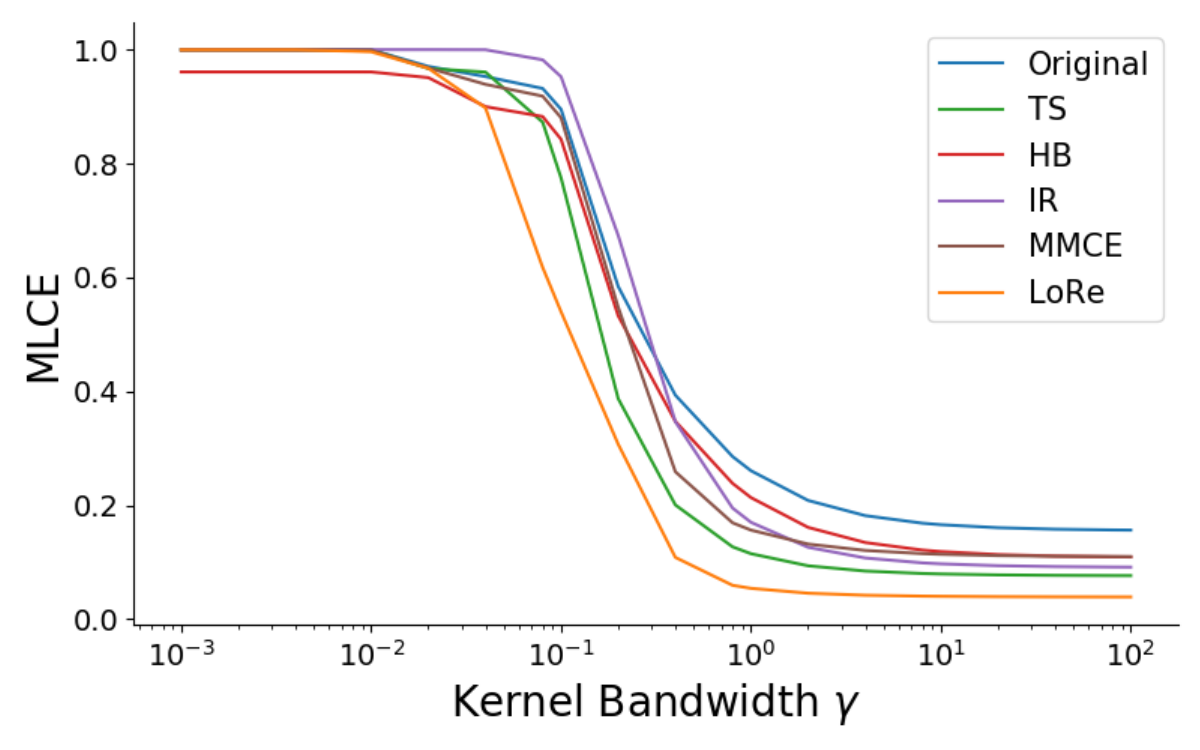}
        \caption{MLCE vs.\ kernel bandwidth $\gamma$ for all recalibration methods on ImageNet using AlexNet features. \method{} achieves the best MLCE for most $\gamma$.}
        \label{fig:mlce_alexnet}
    \end{minipage}
\end{figure}

\begin{figure}[H]
    \begin{minipage}{0.48\textwidth}
        \centering
        \includegraphics[width=\linewidth,trim={0cm 0cm 0cm 0cm}]{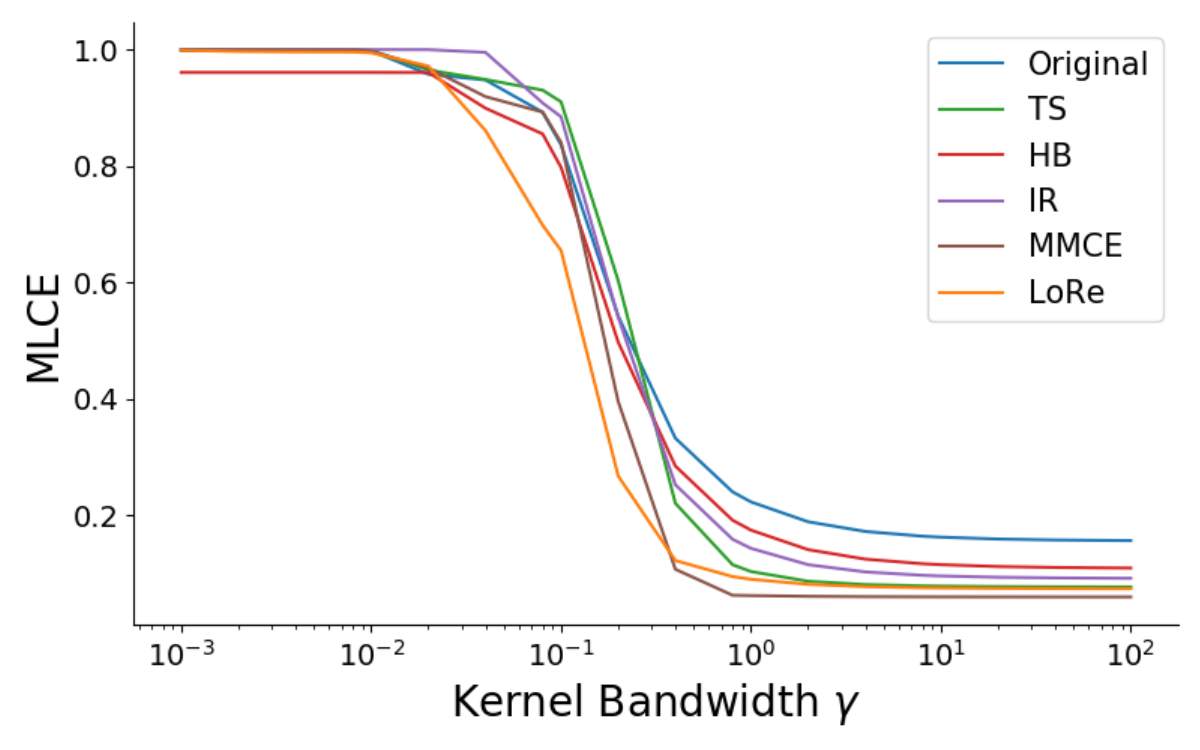}
        \caption{MLCE vs.\ kernel bandwidth $\gamma$ for all recalibration methods on ImageNet using DenseNet121 features. \method{} achieves the best MLCE for most $\gamma$.}
        \label{fig:mlce_densenet121}
    \end{minipage}
    \hfill
    \begin{minipage}{0.48\textwidth}
        \centering
        \includegraphics[width=\linewidth, trim={0cm 0cm 0cm 0cm},clip]{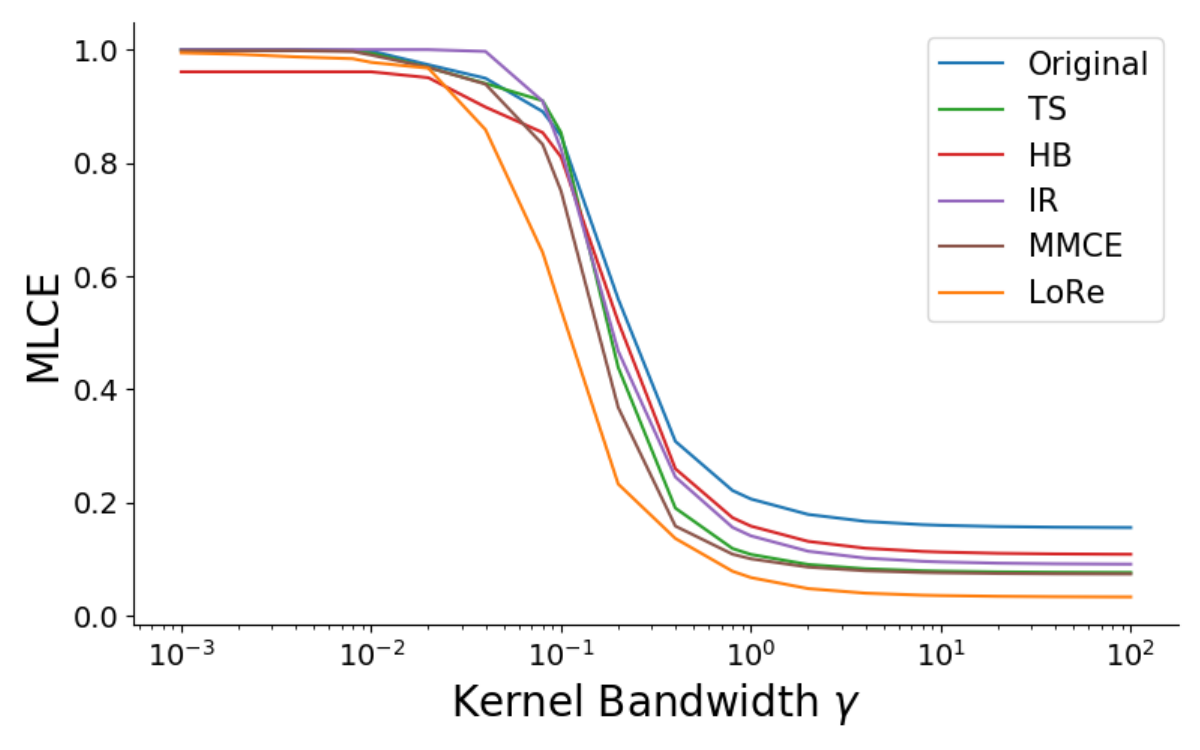}
        \caption{MLCE vs.\ kernel bandwidth $\gamma$ for all recalibration methods on ImageNet using ResNet101 features. \method{} achieves the best MLCE for most $\gamma$.}
        \label{fig:mlce_resnet101}
    \end{minipage}
\end{figure}

\begin{figure}[H]
    \begin{minipage}{0.48\textwidth}
        \centering
        \includegraphics[width=\linewidth,trim={0cm 0cm 0cm 0cm}]{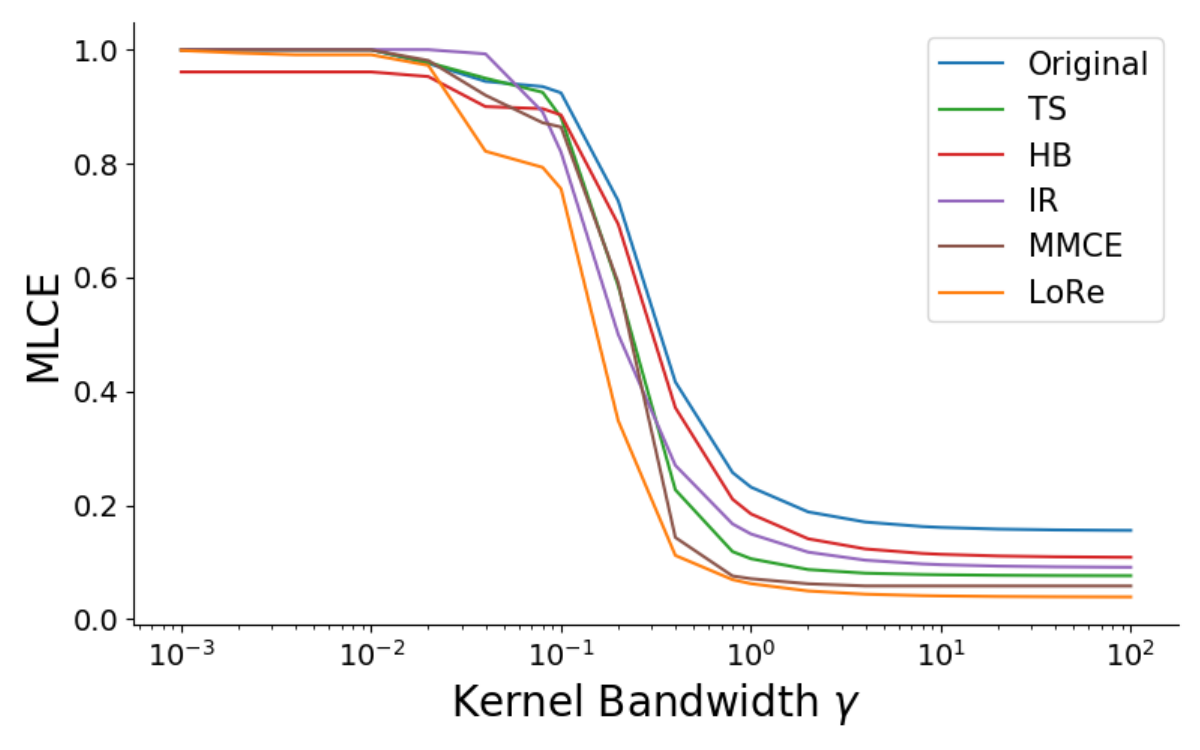}
        \caption{MLCE vs.\ kernel bandwidth $\gamma$ for all recalibration methods on ImageNet using Inception-v3 features to calculate the MLCE and AlexNet features for applying \method{}.}
        \label{fig:mlce_inceptionv3_lore_alexnet}
    \end{minipage}
    \hfill
    \begin{minipage}{0.48\textwidth}
        \centering
        \includegraphics[width=\linewidth, trim={0cm 0cm 0cm 0cm},clip]{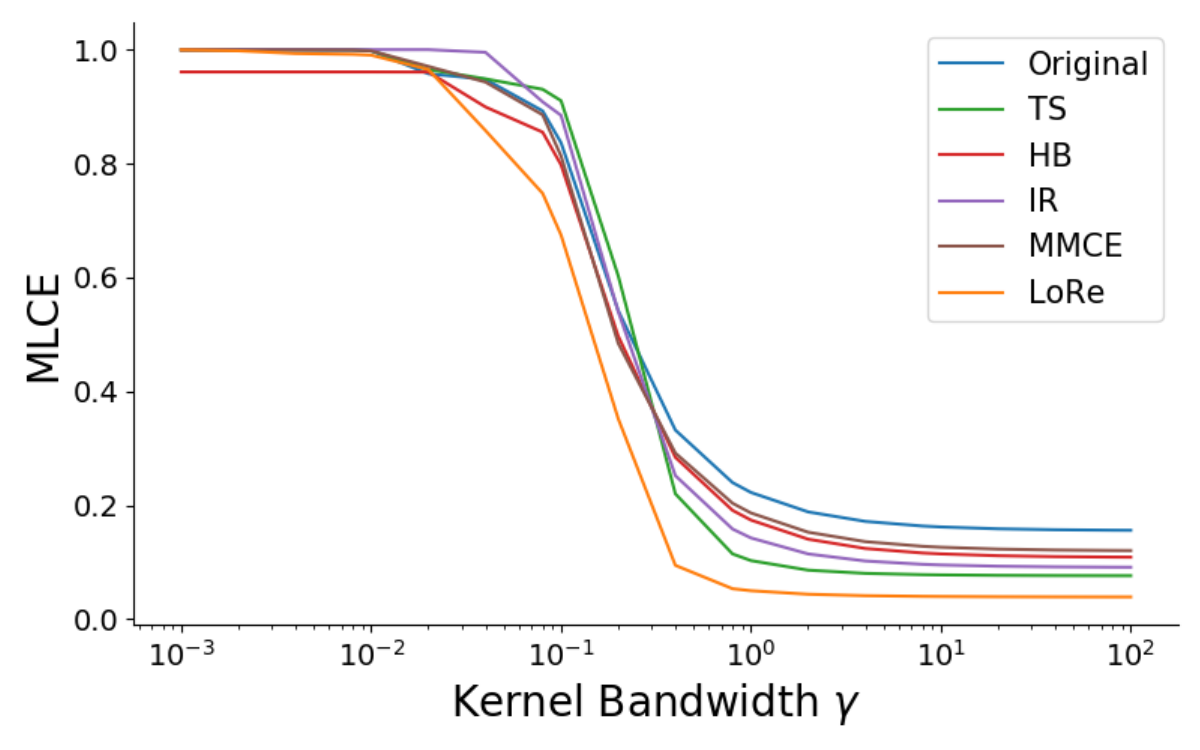}
        \caption{MLCE vs.\ kernel bandwidth $\gamma$ for all recalibration methods on ImageNet using DenseNet121 features to calculate the MLCE and AlexNet features for applying \method{}.}
        \label{fig:mlce_densenet121_lore_alexnet}
    \end{minipage}
\end{figure}

\begin{figure}[H]
    \begin{minipage}{0.48\textwidth}
        \centering
        \includegraphics[width=\linewidth,trim={0cm 0cm 0cm 0cm}]{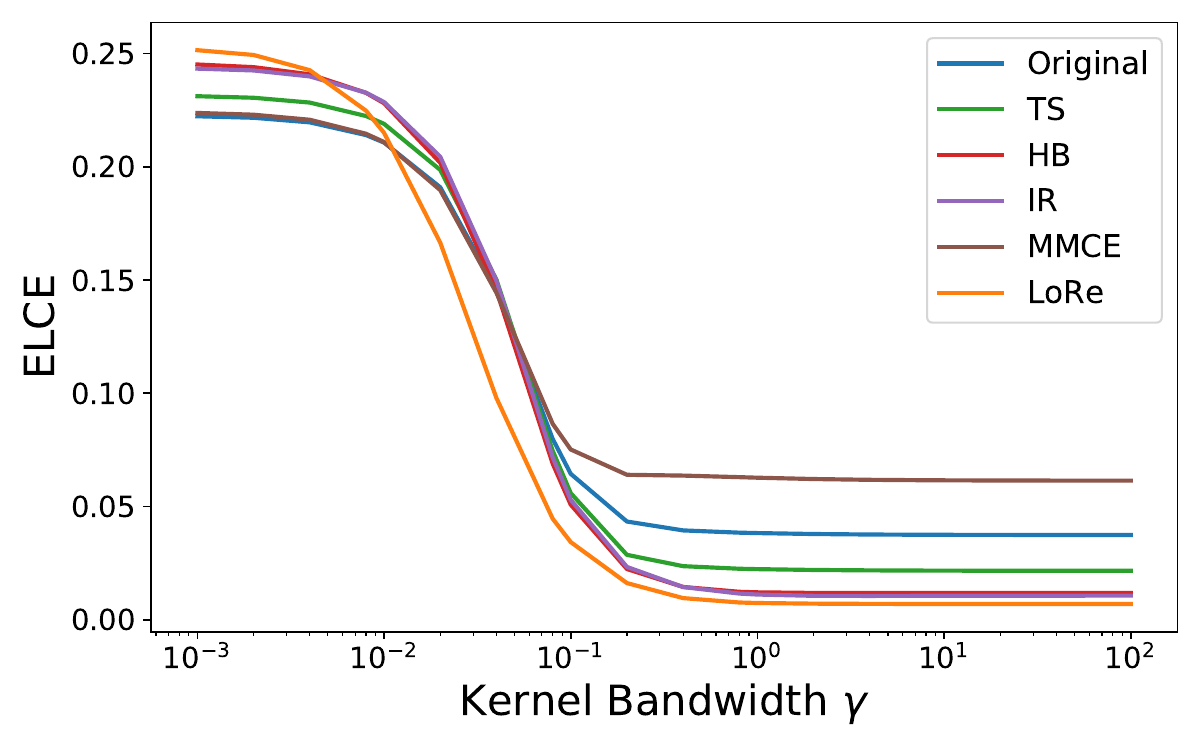}
        \caption{Average LCE vs.\ kernel bandwidth $\gamma$ for all recalibration methods on ImageNet (3D t-SNE features). \method{} gets lower average LCE for most $\gamma$.}
        \label{fig:elce_imagenet}
    \end{minipage}
    \hfill
    \begin{minipage}{0.48\textwidth}
        \centering
        \includegraphics[width=\linewidth, trim={0cm 0cm 0cm 0cm},clip]{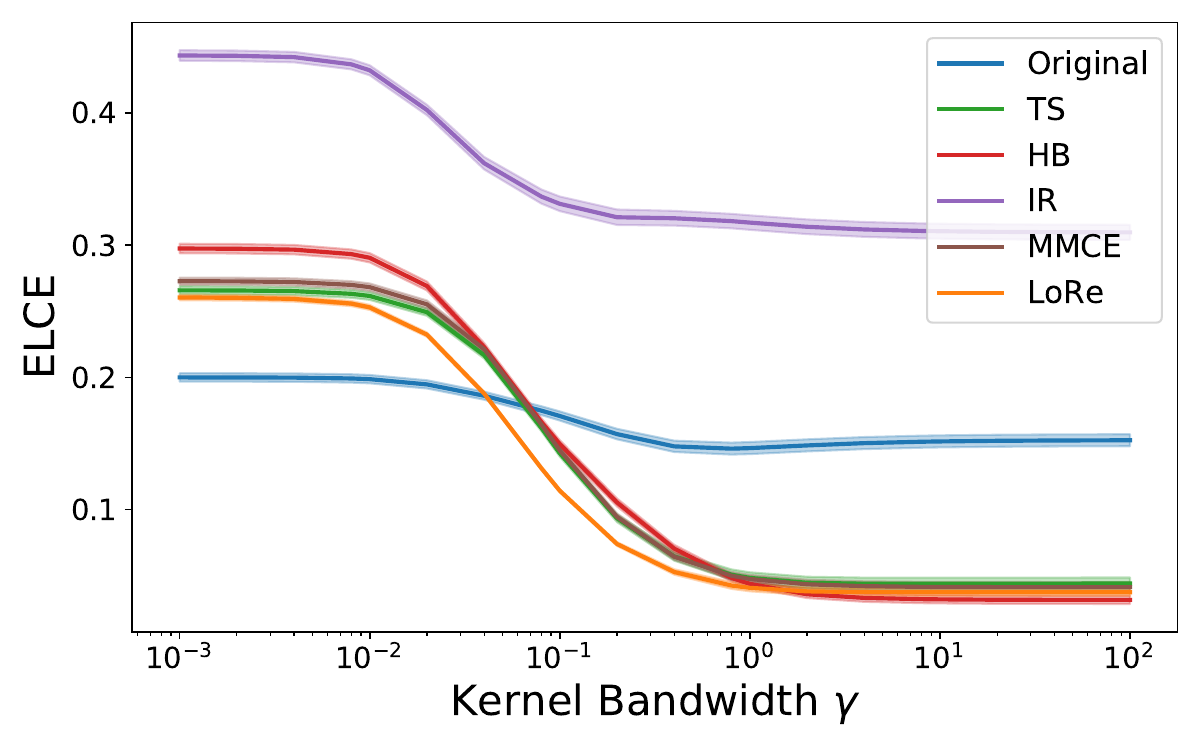}
        \caption{Average LCE vs.\ kernel bandwidth $\gamma$ for all recalibration methods in task 1 (crime data, 2D t-SNE features). \method{} gets lower average LCE for most $\gamma$.}
        \label{fig:elce_setting1}
    \end{minipage}
\end{figure}

\begin{figure}[H]
    \begin{minipage}{0.48\textwidth}
        \centering
        \includegraphics[width=\linewidth,trim={0cm 0cm 0cm 0cm}]{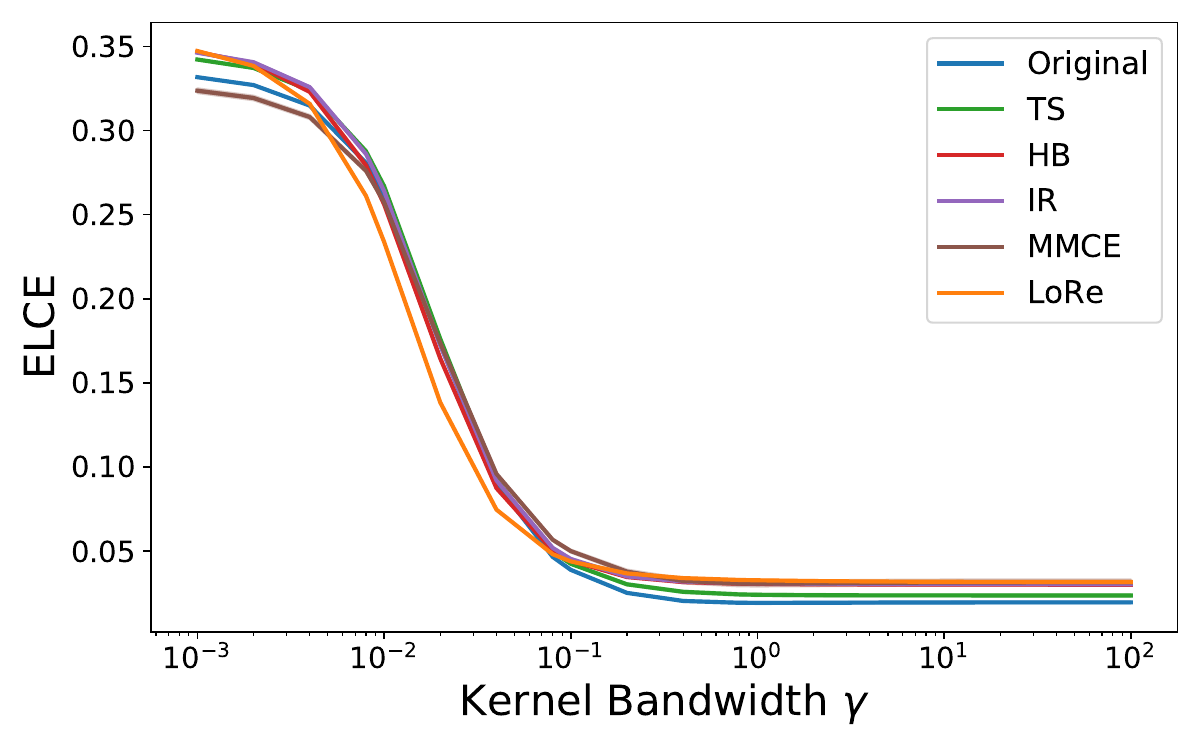}
        \caption{Average LCE vs.\ kernel bandwidth $\gamma$ for all recalibration methods in task 2 (CelebA, 2D t-SNE features).  \method{} gets lower average LCE for most $\gamma$.}
        \label{fig:elce_setting2}
    \end{minipage}
    \hfill
    \begin{minipage}{0.48\textwidth}
        \centering
        \includegraphics[width=\linewidth, trim={0cm 0cm 0cm 0cm},clip]{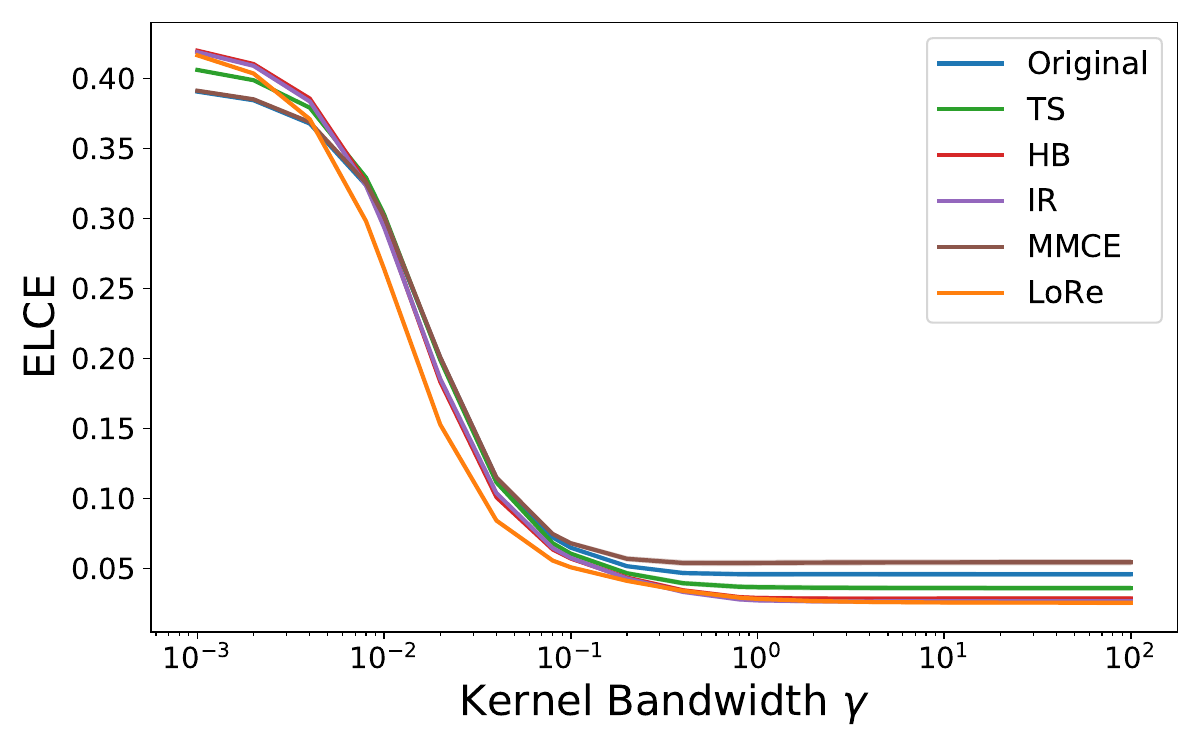}
        \caption{Average LCE vs.\ kernel bandwidth $\gamma$ for all recalibration methods in task 3 (CelebA, 2D t-SNE features). \method{} gets lower average LCE for most $\gamma$.}
        \label{fig:elce_setting3}
    \end{minipage}
\end{figure}

\section{Proof of Lemma \ref{lemma:global}}
\label{appendix:lemma}
We restate Lemma \ref{lemma:global} below, and provide the proof:
\begin{lemma}
Assume that $\lim_{\gamma \to \infty} k_{\gamma}(x, x') = 1$ for all $x, x' \in \mc{X}$. Then, as $\gamma \to \infty$, the MLCE converges to the MCE.
\end{lemma}
\begin{proof}
Since $\lim_{\gamma \to \infty} k_{\gamma}(x, x') = 1$ identically,
\begin{align*}
    \lim_{\gamma \to \infty} \max_x \widehat{\mathrm{LCE}}_{\gamma}(x; f, \hp)
    &= \max_x \frac{1}{\abs{\beta(x)}} \abs{\sum_{i \in \beta(x)} \hp(x_i) - \indic{f(x_i) = y_i}} \\
    &= \max_k \frac{1}{\abs{B_k}} \abs{\sum_{i \in B_k} \hp(x_i) - \indic{f(x_i) = y_i}} \\
    &= \max_k \abs{\mathrm{conf}(B_k) - \mathrm{acc}(B_k)} \\
    &= \mathrm{MCE}(x; f, \hp)
\end{align*}
\end{proof}

\section{Formal Statement and Proof of Theorem \ref{theorem:slce_informal}}
\label{appendix:lce}

Let $B_1,\dots,B_N$ denote a set of bins that partition $[0,1]$, and $B(p)$ denote the bin that a particular $p\in[0,1]$ belongs to. Let $a_f(x, y) = \indic{f(x) = y}$ indicate the accuracy of a the classifier $(f, \hp)$ on an input $x$. We consider the signed local calibration error (SLCE):
\begin{align*}
    \quad \slce_{\gamma}(x; f, \hp) &:=
    \frac{\Eb[(\hp(X) - a_f(X, Y)) k_{\gamma}(X, x) \mid \hp(X) \in B(\hp(x))]}{\Eb[k_{\gamma}(X, x) \mid \hp(X) \in B(\hp(x))]} \\
    & \; = \frac{\Eb[(\hp(X) - a_f(X, Y)) k_{\gamma}(X, x) \indic{\hp(X) \in B(\hp(x))}]}{\Eb[k_{\gamma}(X, x) \indic{\hp(X) \in B(\hp(x))}]}.
\end{align*}

\subsection{Assumptions and Formal Statement of Theorem}
We make the following assumptions:

\begin{assumption}[Lipschitz kernel]
\label{assumption:kernel}
The kernel $k_\gamma$ takes the form
\begin{align*}
    k_\gamma(x, x') = g\paren{ \frac{\phi(x) - \phi(x')}{\gamma}},
\end{align*}
where $\phi:\mc{X}\to \R^d$ is a representation function, and $g:\R^d\to [0,1]$ is $L$-Lipschitz with respect to some norm $\norm{\cdot}$.
\end{assumption}
Note this definition may require an implicit rescaling (for example, we can take $\phi(x)\leftarrow \phi^{\rm feature}(x)/d$ for a $d$-dimensional feature map $\phi^{\rm feature}$ and take $g(z)=\exp(-\|z\|_1)$, which corresponds to the Laplacian kernel we used in Section~\ref{section:choice-kernel}).

\begin{assumption}[Binning-aware covering number]
\label{assumption:cover}
For any $\epsilon>0$, the range of the representation function $\phi(\mc{X}):= \set{\phi(x): x \in \mc{X}}$ has an $\epsilon$-cover in the $\norm{\cdot}$-norm of size $(C/\epsilon)^d$ for some absolute constant $C>0$: There exists a set $\mc{N}_\epsilon\in\mc{X}$ with $|\mc{N}_\epsilon|\le (C/\epsilon)^d$ such that for any $x\in\mc{X}$, there exists some $x'\in \mc{N}_\epsilon$ such that $\norm{\phi(x)-\phi(x')}\le \epsilon$ and $B(\hp(x))=B(\hp(x'))$.
\end{assumption}

\begin{assumption}[Lower bound on expectation of kernel within bin]
\label{assumption:lower-bound}
We have
\begin{align*}
    \inf_{x\in\mc{X}} \E\brac{k_\gamma(X, x) \indic{\hp(X)\in B(\hp(x))}} \ge \alpha
\end{align*}
for some constant $\alpha\in(0,1)$. 
\end{assumption}
The constant $\alpha$ characterizes the hardness of estimating the SLCE from samples. Intuitively, with a smaller $\alpha$, the denominator in SLCE gets smaller and we desire a higher accuracy in estimating both the numerator and the denominator. Also note that in practice the value of $\alpha$ typically depends on $\gamma$.

We analyze the following estimator of the SLCE using $n$ samples:
\begin{align}
    \slcehat_\gamma(x; f, \hp) = \frac{ \frac{1}{n}\sum_{i=1}^n (\hp(x_i) - a_f(x_i, y_i))k_\gamma(x_i, x) \indic{\hp(x_i) \in B(\hp(x))} }{ \frac{1}{n}\sum_{i=1}^n k_\gamma(x_i, x) \indic{\hp(x_i) \in B(\hp(x))}  }.
\end{align}

\begin{theorem}
\label{theorem:slce}
Under Assumptions~\ref{assumption:kernel},~\ref{assumption:cover}, and~\ref{assumption:lower-bound}, Suppose the sample size $n\ge \widetilde{O}(d/\alpha^4\epsilon^2)$ where $\epsilon>0$ is a target accuracy level, then with probability at least $1-\delta$ we have
\begin{align*}
    \sup_{x\in\mc{X}} \abs{\slcehat_\gamma(x; f, \hp) - \slce_\gamma(x; f, \hp) } \le \epsilon,
\end{align*}
where $\widetilde{O}$ hides log factors of the form $\log(L/\gamma\epsilon\delta \alpha)$.
\end{theorem}
Theorem~\ref{theorem:slce} shows that $\wt{O}(d/\epsilon^2\alpha^4)$ samples is sufficient to estimate the SLCE simultaneously for all $x\in\mc{X}$. When $\alpha=\Omega(1)$, this sample complexity only depends polynomially in terms of the representation dimension $d$ and logarithmically in other constants (such as $L,\gamma$, and the failure probability $\delta$).

\subsection{Proof of Theorem~\ref{theorem:slce}}

{\bf Step 1.} We first study the estimation at finitely many $x$'s. Let $\mc{N}\subseteq\mc{X}$ be a finite set of $x$'s with $|\mc{N}|=N$. Since $k_\gamma\in[0,1]$ and $|\hp(x)-a_f(x, y)|\le 1$ are bounded variables, by the Hoeffding inequality and a union bound, we have
\begin{align*}
    & \quad \P\bigg( \sup_{x \in \mc{N}} \bigg| \frac{1}{n}\sum_{i=1}^n (\hp(x_i) - a_f(x_i, y_i))k_\gamma(x_i, x) \indic{\hp(x_i) \in B(\hp(x))} \\
    & \qquad \qquad - \E\brac{ (\hp(X) - a_f(X, Y))k_\gamma(X, x) \indic{\hp(X) \in B(\hp(x)) } }\bigg| > \alpha\epsilon/10\bigg) \\
    & \le \exp(-cn\alpha^2\epsilon^2 + \log N).
\end{align*}
Therefore, as long as $n\ge O(\log (N/\delta)/\epsilon^2\alpha^2)$ samples, the above probability is bounded by $\delta$. In other words, with probability at least $1-\delta$, we have simultaneously
\begin{align*}
    & \bigg| 
    \underbrace{\frac{1}{n}\sum_{i=1}^n (\hp(x_i) - a_f(x_i, y_i))k_\gamma(x_i, x) \indic{\hp(x_i) \in B(\hp(x))}}_{:=\hat{A}(x)} \\
    & \qquad - \underbrace{\E\brac{ (\hp(X) - a_f(X, Y))k_\gamma(X, x) \indic{\hp(X) \in B(\hp(x)) } }}_{:=A(x)} \bigg| \\
    & \le  \alpha\epsilon/10.
\end{align*}
for all $x\in \mc{N}$. Similarly, when $n\ge O(\log(N/\delta)/\epsilon^2\alpha^4)$, we also have (with probability at least $1-\delta$)
\begin{align*}
    \bigg| 
    \underbrace{\frac{1}{n}\sum_{i=1}^n k_\gamma(x_i, x) \indic{\hp(x_i) \in B(\hp(x))}}_{:= \hat{B}(x)} - \underbrace{\E\brac{ k_\gamma(X, x) \indic{\hp(X) \in B(\hp(x)) } }}_{:= B(x)}
    \bigg| \le  \alpha^2\epsilon/10
\end{align*}
On these concentration events, we have for any $x\in \mc{N}$ that
\begin{align*}
    \quad \abs{\slcehat_\gamma(x; f, \hp) - \slce_\gamma(x; f, \hp)} 
    & = \abs{ \frac{\hat{A}(x)}{\hat{B}(x)} - \frac{A(x)}{B(x)} } \\ 
    & \le \abs{\hat{A}(x)} \abs{\frac{1}{\hat{B}(x)} - \frac{1}{B(x)} } + \frac{1}{\abs{B(x)}} \abs{\hat{A}(x) - A(x)} \\
    & \le 1\cdot \frac{\alpha^2 \epsilon/10}{\alpha (\alpha - \alpha^2\epsilon/10)} + \frac{1}{\alpha} \cdot \alpha\epsilon/10 \\
    & \le \epsilon.
\end{align*}

{\bf Step 2.} We now extend the bound to all $x\in\mc{X}$ using the covering argument. By Assumption~\ref{assumption:cover}, we can take an $\alpha^2\epsilon\gamma/(10L)$-covering of $\phi(\mc{X})$ with cardinality $N\le (10CL/\alpha^2\epsilon\gamma)^d$. Let $\mc{N}\subset \mc{X}$ denote the covering set (in the $\mc{X}$ space). This means that for any $x\in\mc{X}$, there exists $x'\in\mc{N}$ such that $\norm{\phi(x) - \phi(x')}\le \alpha^2\epsilon\gamma/(10L)$ amd $B(\hp(x))=B(\hp(x'))$, which implies that for any $\wt{x}\in\mc{X}$ we have
\begin{align*}
    \quad \abs{k(\wt{x}, x) - k(\wt{x}, x')} 
    & = \abs{f\paren{\frac{\phi(\wt{x}) - \phi(x)}{\gamma}} - f\paren{\frac{\phi(\wt{x}) - \phi(x')}{\gamma}}} \\
    & \le \frac{L}{\gamma} \norm{\phi(x) - \phi(x')} \\
    & \le \alpha^2\epsilon/10,
\end{align*}
where we have used the Lipschitzness assumption of $g$ (Assumption~\ref{assumption:kernel}). This further implies
\begin{align*}
    \quad \abs{\hat{A}(x) - \hat{A}(x')} 
    & = \biggl| \frac{1}{n}\sum_{i=1}^n (\hp(x_i) - a_f(x_i, y_i))k_\gamma(x_i, x) \indic{\hp(x_i) \in B(\hp(x))} \\
    & \qquad \qquad - \frac{1}{n}\sum_{i=1}^n (\hp(x_i) - a_f(x_i, y_i))k_\gamma(x_i, x') \indic{\hp(x_i) \in B(\hp(x'))} \biggr| \\
    & = \biggl| \frac{1}{n}\sum_{i=1}^n (\hp(x_i) - a_f(x_i, y_i))\brac{k_\gamma(x_i, x) - k_\gamma(x_i, x')} \indic{\hp(x_i) \in B(\hp(x))} \biggr| \\
    & \le \frac{1}{n}\sum_{i=1}^n \abs{\hp(x_i) - a_f(x_i, y_i)} \cdot \abs{k_\gamma(x_i, x) - k_\gamma(x_i, x')} \cdot \indic{\hp(x_i) \in B(\hp(x))} \\
    & \le \alpha^2\epsilon/10.
\end{align*}
Similarly, we have $|A(x) - A(x')|\le \alpha^2\epsilon/10$, $|\hat{B}(x) - \hat{B}(x')|\le \alpha^2\epsilon/10$, and $|B(x) - B(x')|\le \alpha^2\epsilon/10$. This means that the estimation error at $x$ is close to that at $x'\in\mc{N}$ and consequently also bounded by $\epsilon$:
\begin{align*}
    \abs{\slcehat_\gamma(x; f, \hp) - \slce_\gamma(x; f, \hp)} &= \abs{\frac{\hat{A}(x)}{\hat{B}(x)} - \frac{A(x)}{B(x)} } \\
    & \le \abs{ \frac{\hat{A}(x)}{\hat{B}(x)} - \frac{\hat{A}(x')}{\hat{B}(x')} } + \abs{ \frac{\hat{A}(x')}{\hat{B}(x')} - \frac{A(x')}{B(x')} } + \abs{ \frac{A(x')}{B(x')} - \frac{A(x)}{B(x)} } \\
    & \le 3 \brac{ 1\cdot \frac{\alpha^2 \epsilon/10}{\alpha (\alpha - \alpha^2\epsilon/10)} + \frac{1}{\alpha} \cdot \alpha^2\epsilon/10 } \\
    & \le \epsilon.
\end{align*}
Therefore, taking this $\mc{N}$ in step 1, we know that as long as the sample size
\begin{align*}
    N \ge O\paren{ \frac{\log(|\mc{N}|/\delta)}{\epsilon^2\alpha^4} } = O\paren{ \frac{d\brac{\log(10CL/\alpha^2\epsilon\gamma) + \log(1/\delta)}}{\alpha^4\epsilon^2} } = \wt{O}\paren{d/\alpha^4\epsilon^2},
\end{align*}
we have with probability at least $1-\delta$ that
\begin{align*}
    \sup_{x\in\mc{X}} \abs{\slcehat_\gamma(x; f, \hp) - \slce_\gamma(x; f, \hp)} \le \epsilon.
\end{align*}
This is the desired result.

\qed


\end{document}